%% file: main_arxiv.tex
\documentclass[10pt]{article}
\usepackage[preprint]{tmlr}

\input{preamble}

\title{Causal Optimizer Interaction Calculus:\\
Hidden Geometric Relaxation and Identifiable Interventions}

\author{\name Zavier Li \email zavierli888@gmail.com\\
      \addr Xidian University\\
      \addr Xi'an, China}

\def\month{07}
\def\year{2026}
\def\openreview{\url{}}

\begin{document}
\maketitle

\begin{abstract}
An optimizer experiment observes responses to algorithmic configurations without uniquely revealing hidden mechanisms.  We develop a causal optimizer interaction calculus that separates pathwise realization, M\"obius decomposition, and experimental identification.  Under a fixed innovation coupling, every finite-horizon innovation-driven optimizer admits a canonical behaviorally minimal pathwise realization.  Over any lower-finite intervention poset, its expected response has unique pure effects; for every finite support and design, an incidence operator gives the complete observational gauge, exact identifiability, sharp quotient stability, held-out predictions, and exact noiseless configuration complexity.  The companion P1 paper proves that smooth hidden relaxation gives reduced-value curvature $-G^*H^{-1}G$ and that Boolean contrasts integrate it.  Taking this as a structural input, we prove an observable-readout transfer theorem: arbitrary smooth update or trace readouts inherit an explicit five-term interaction through first and second hidden responses, with no universal sign beyond the reduced-value case.  A third theorem develops Gaussian quotient minimax risk, exact confidence sets and tests, misspecification decomposition, certified decisions, and optimal replication.  A controlled real-data experiment closes the reduced-value chain on a 65-dimensional strongly convex logistic model.  Boolean effects and independent curvature integrals agree within $4.21\times10^{-11}$; nine unobserved continuous intensities agree within $8.88\times10^{-13}$; 100,000 Gaussian campaigns attain 0.9497 ellipsoid coverage and 0.9626 held-out power; and 4500 real-minibatch observations with 20,000 bootstrap draws reject order-two support with interval $[0.001839,0.004399]$.  Finite-response error decreases from $1.43\times10^{-2}$ to $5.55\times10^{-17}$ within its strong-convity bound.  Neural trace audits provide complementary nonconvex response-class evidence.  The paper contributes the causal, readout, and identifiable experimental layer while assigning the underlying Schur and path-space laws to P1 and P2.
\end{abstract}

\input{sections/01_introduction}
\input{sections/02_related_work}
\input{sections/03_causal_language}
\input{sections/04_universal_representation}
\input{sections/05_hidden_relaxation}
\input{sections/06_observation_design}
\input{sections/07_geometric_response_classes}
\input{sections/08_experiments}
\input{sections/09_discussion}
\input{sections/10_statements}

\bibliographystyle{tmlr}
\bibliography{references}

\appendix
\input{appendix/a_universal_proof}
\input{appendix/b_interaction_proof}
\input{appendix/c_observation_proof}
\input{appendix/d_geometry_proofs}
\input{appendix/e_experiment_details}

\end{document}

%% file: preamble.tex
\usepackage{amsmath,amssymb,amsthm,mathtools}
\usepackage{bm}
\usepackage{booktabs}
\usepackage{array}
\usepackage{graphicx}
\usepackage{enumitem}
\usepackage{microtype}
\usepackage{url}
\usepackage{hyperref}
\usepackage[nameinlink,noabbrev]{cleveref}

\theoremstyle{plain}
\newtheorem{theorem}{Theorem}
\newtheorem{proposition}{Proposition}

\newtheorem{corollary}{Corollary}
\theoremstyle{definition}

\newtheorem{assumption}{Assumption}
\theoremstyle{remark}

\newcommand{\R}{\mathbb R}
\newcommand{\Hspace}{\mathbb H}
\newcommand{\Pcal}{\mathcal P}
\newcommand{\Kcal}{\mathcal K}
\newcommand{\Dcal}{\mathcal D}
\newcommand{\Acal}{\mathcal A}
\newcommand{\Mcal}{\mathcal M}
\newcommand{\Fcal}{\mathcal F}
\newcommand{\Ocal}{\mathcal O}
\newcommand{\Rcal}{\mathcal R}
\newcommand{\Tcal}{\mathcal T}
\newcommand{\SPD}{\mathbb S_{++}}

\newcommand{\rank}{\operatorname{rank}}
\newcommand{\tr}{\operatorname{tr}}
\newcommand{\diag}{\operatorname{diag}}

\newcommand{\Proj}{\operatorname{Proj}}
\newcommand{\range}{\operatorname{range}}
\newcommand{\VarOp}{\operatorname{Var}}

\newcommand{\dd}{\mathrm d}
\newcommand{\EE}{\mathbb E}
\newcommand{\argminop}{\operatorname*{arg\,min}}
\newcommand{\norm}[1]{\left\lVert #1\right\rVert}
\newcommand{\opnorm}[1]{\left\lVert #1\right\rVert_{\mathrm{op}}}
\newcommand{\ip}[2]{\left\langle #1,#2\right\rangle}

\newcommand{\TRC}{\operatorname{TRC}}
\newcommand{\TGER}{\operatorname{TGER}}

\crefname{assumption}{Assumption}{Assumptions}
\crefname{definition}{Definition}{Definitions}
\crefname{theorem}{Theorem}{Theorems}
\crefname{proposition}{Proposition}{Propositions}
\crefname{lemma}{Lemma}{Lemmas}
\crefname{corollary}{Corollary}{Corollaries}

%% file: sections/01_introduction.tex
\section{Introduction}
\label{sec:introduction}

An optimizer is an intervention-dependent dynamical system.  Its state may contain parameters, momentum, curvature estimates, preconditioners, schedules, clipping rules, random innovations, target transformations, and numerical state.  A recorded parameter update is one projection of that extended dynamics.  It generally admits many internal explanations: positive geometry may reproduce part of the update, while memory, operators, control, noise, or target changes reproduce the remainder.

This ambiguity creates two distinct scientific questions.  A \emph{response-class question} asks whether an observed trace is close to the outputs generated by a declared family, such as bounded diagonal geometry.  A \emph{mechanism question} asks how the counterfactual response changes when a declared module is altered.  Passive projection can reject a response class.  Causal mechanism effects require interventions, instrumentation, or an observation structure with equivalent identifying power.

We take the complete interventional response as the primitive object.  A lower-finite poset $(\Pcal,\preceq)$ indexes configurations, including Boolean sector masks, graded mechanism levels, and nested geometry budgets.  Under a protocol that fixes initialization, data, horizon, response map, and noise coupling, each configuration $a$ produces a finite-horizon response $F(a)$ in a Hilbert space.  The optimizer may be nonsmooth, stochastic, stateful, delayed, and history-dependent.

Three operations then have separate roles.  Causal realization turns histories into a predictive state.  M\"obius inversion changes coordinates from configuration responses to pure interventional effects.  A linear design operator maps those effects to queried observations.  Their order is typed:
\[
\text{causal dynamics}\longrightarrow
\text{response map}\longrightarrow
\text{M\"obius effects}\longrightarrow
\text{observational quotient}.
\]
The universal theory requires no optimizer energy or geometry.

The companion P1 paper already answers the structural reduced-value question.  When intervention amplitudes couple to hidden optimizer variables that relax by energy minimization, it proves that the reduced intervention Hessian is
\[
D^2_{uu}\bar E=-G^*H^{-1}G,
\]
and that Boolean reduced-value contrasts integrate this curvature \citep{li2026optimizationgeometrodynamics}.  P2 supplies the corresponding path-space M\"obius--Jacobi theorem \citep{li2026restrictedcomplexity}.  The open problem for an experimental calculus is that optimizer studies usually observe updates, traces, or task statistics rather than the reduced optimal value itself.

We solve this interface problem with an observable readout theorem.  For any smooth response $F_R(u)=R(u,\sigma_\star(u))$, its mixed derivative is an explicit five-term transfer through the first and second hidden responses.  The P1 negative Gram law is recovered for the reduced-value readout, while an actual update readout generally has no universal sign.  Its Boolean M\"obius effect remains the exact integral of this transferred curvature and is therefore accessible to a factorial design.  This separates three statements that had previously been conflated: P1 explains reduced-value curvature, the new transfer law predicts the chosen observable, and the incidence experiment determines whether that observable effect is identifiable.

This positions the paper as the experimental outer layer of a four-paper program.  Information-induced training geometry supplies explicit visible metrics and minimal SPD completion \citep{li2026informationinducedgeometry}.  P1 supplies finite-dimensional hidden-state reduction and reduced-value curvature \citep{li2026optimizationgeometrodynamics}.  P2 supplies structured path-space geometry, Jacobi response, and its global M\"obius realization \citep{li2026restrictedcomplexity}.  The present paper contributes the pathwise causal state, observable-readout transfer, experimental gauge, inference, and design layers.  Its universal representation theorem does not require the regular companion theories.

\paragraph{Contributions.}
\begin{enumerate}[leftmargin=*,itemsep=3pt]
\item We prove a universal causal--interventional representation theorem.  Under the protocol's fixed innovation coupling, pathwise predictive equivalence gives the canonical behaviorally minimal optimizer state.  Lower-finite incidence coordinates give unique pure effects, the complete design gauge, exact identifiability, sharp quotient stability, held-out falsification, and exact noiseless configuration complexity for every finite support.
\item Taking the P1 Schur--M\"obius theorem as a structural input, we prove an observable readout interaction law.  It transfers first and second hidden relaxation responses to arbitrary smooth scalar readouts, covers actual update coordinates, integrates exactly to their Boolean effects, and shows why the reduced-value sign law does not extend to general optimizer traces.
\item We develop the identifiable observation layer: projection certificates, Gaussian quotient minimax risk, exact confidence sets, residualized sector tests, misspecification decomposition, certified downstream actions, total-recovery allocation, and exact target-subspace A-optimal replication.  The unrestricted Boolean recovery risk is exponential, quantifying the price of universal attribution.
\item We complete a controlled real-data factorial closure on a 65-dimensional strongly convex logistic hidden problem.  Independent Boolean and curvature paths, continuous held-out intensities, eight data splits, information-optimal allocation, 100,000 Gaussian campaigns, 4500 real-minibatch observations, 20,000 bootstrap draws, held-out support rejection, and finite-response bounds validate the reduced-value structural input and the identification/inference layer.  Separate neural audits instantiate full-SPD, diagonal, block, channel-scalar, layer-scalar, and trajectory-restricted response classes in nonconvex training.
\end{enumerate}

\paragraph{Scope.}
Universality concerns pathwise causal representation under a fixed innovation coupling and incidence identification of expected responses.  Smooth readout interaction requires the stated regularity, uniqueness, and positive vertical-Hessian assumptions.  Exact Gaussian inference additionally assumes a finite-dimensional response and known positive-definite covariance.  Trace explainability is not an optimizer-performance ranking.

%% file: sections/02_related_work.tex
\section{Related Work and Series Position}
\label{sec:related-work}

\paragraph{Optimizer geometry and structured preconditioning.}
Riemannian optimization, natural gradient, and mirror descent formalize metric or dual-geometric update rules \citep{absil2008optimization,amari1998natural,beck2003mirror,boumal2023introduction}.  Numerical optimization develops variable metrics, quasi-Newton updates, line search, and trust regions \citep{nocedal2006numerical,combettes2014variable,parikh2014proximal}.  AdaGrad, Adam, K-FAC, Shampoo, Adafactor, SM3, and SOAP instantiate diagonal, block, Kronecker, factored, compressed, or moving-basis restrictions \citep{duchi2011adagrad,kingma2014adam,martens2015kfac,gupta2018shampoo,shazeer2018adafactor,anil2019sm3,vyas2025soap}.  We treat each declared family as a response class or intervention sector.

\paragraph{Dynamic and learned optimizers.}
Feedback-system analysis, performance estimation, learning to optimize, and algorithm selection study optimizer behavior, synthesis, or benchmarking \citep{lessard2016analysis,drori2014performance,taylor2017smooth,andrychowicz2016learning,metz2022velo,kerschke2019automated,dahl2023algoperf}.  Our object is the counterfactual response law over declared configurations and the information available for identifying its effects.

\paragraph{Causal realization and system identification.}
Predictive equivalence follows the minimal-realization principle that two histories are merged exactly when no admissible future distinguishes them \citep{nerode1958linear}.  Predictive-state representations encode dynamical state directly through predictions of future observable tests \citep{singh2004predictive}.  System identification reconstructs dynamics from controlled input--output observations \citep{ljung1999system}, while inverse optimization recovers latent objectives or parameters from decisions \citep{ahuja2001inverse}.  Our finite-horizon theorem uses the stronger pathwise equivalence induced by the protocol's fixed innovation coupling; this yields an exact right congruence, then connects the quotient to a finite-poset intervention gauge.  It does not claim a distribution-only minimal stochastic realization.

\paragraph{M\"obius effects and factorial experiments.}
M\"obius inversion in incidence algebras is classical.  Potential-outcome analyses of $2^K$ factorial designs define causal main and interaction effects and their randomization-based inference \citep{dasgupta2015causal}.  Shapley--Taylor indices and Integrated Hessians connect discrete feature interactions to derivative integrals under different axioms and baselines \citep{dhamdhere2020shapley,janizek2021explaining}.  Randomization tests provide exact finite-sample calibration under invariant nulls \citep{hemerik2018exact,phipson2010permutation}.  Relative to classical factorial potential outcomes, we allow lower-finite posets and declared sparse supports, characterize the complete observation kernel and exact noiseless configuration complexity, and connect observable effects to a hidden optimizer readout law.  We do not claim M\"obius inversion or factorial causal contrasts themselves as new.

\paragraph{Parametric optimization sensitivity.}
Envelope and Schur-complement sensitivity formulas are standard in perturbation analysis and variable projection \citep{golub1973differentiation,bonnans2000perturbation}.  Differentiable optimization layers, implicit differentiation, and bilevel learning compute related solution-map responses for optimization-defined modules \citep{amos2017optnet,blondel2022efficient,franceschi2018bilevel}.  P1 specializes these ideas to hidden optimizer reduction, proves the negative-semidefinite affine-intervention Gram law, its factorial integral, and the noncommutation example restated here \citep{li2026optimizationgeometrodynamics}; P2 proves the path-space M\"obius--Jacobi counterpart \citep{li2026restrictedcomplexity}.  Our new smooth result differentiates a general observable readout through the hidden solution map, exposing the additional terms that govern actual update interactions.  The hidden factorization remains structural and is not identified from the discrete effect alone.

\paragraph{Experimental design.}
Classical optimal-design theory chooses measurements through information criteria and approximate design weights \citep{pukelsheim2006optimal}.  Our incidence experiment specializes this perspective to optimizer intervention masks: rank gives exact effect identifiability, residualized columns give sector tests, and the saturated Boolean case admits a closed-form continuous A-optimal replication allocation.  Integer rounding remains a separate design problem.

\paragraph{Mechanisms beyond positive geometry.}
Momentum, acceleration, clipping, signs, stochastic oracles, proximal maps, and schedules can produce behavior outside a positive parameter-space cometric applied to the current objective covector \citep{polyak1964some,nesterov1983method,su2016differential,wibisono2016variational,liu2023sophia,defazio2024road}.  Muon-style rules apply matrix operators to momentum \citep{jordan2024muon,essentialai2025muon,crawshaw2025muonvariants}.  Our protocol keeps such mechanisms as explicit sectors unless a different enlarged state geometry is declared.

\paragraph{Dependency within the series.}
The pathwise causal representation, M\"obius identification, observable-readout transfer, and statistical results here are self-contained.  The companion information-geometry paper supplies the explicit SPD metric-submetry realization \citep{li2026informationinducedgeometry}; P1 supplies the finite-dimensional Schur and reduced-value interaction theorems \citep{li2026optimizationgeometrodynamics}; and P2 supplies path-space value functions, Jacobi response, and the global M\"obius--Jacobi theorem \citep{li2026restrictedcomplexity}.  P3 imports those structural laws and contributes the experimental interface rather than republishing them as new results.

%% file: sections/03_causal_language.tex
\section{Three Equations for Causal Optimizer Interaction}
\label{sec:language}

\subsection{Causal response}

Let $(\Pcal,\preceq)$ be lower-finite: every principal ideal is finite.  For each configuration $a\in\Pcal$, a finite-step causal optimizer experiment admits a history-state realization
\begin{equation}
x_{k+1}^{a}=\Phi_k^a\!\left(x_k^a,\Ocal_f(x_k^a),\xi_k\right),
\qquad
F(a)=\EE\,\Rcal(x_{0:N}^a)\in\Hspace.
\label{eq:causal-response}
\end{equation}
The protocol fixes initialization, data law, horizon, response map, upstream state, and coupling of exogenous randomness.  The configuration changes only declared mechanisms.  The response may contain update traces, logged states, terminal parameters, and task observables.  Histories, inputs, exposed innovations, and responses are assumed standard Borel.  The fixed innovation coupling is part of the protocol.

An \emph{admissible pathwise realization} has a reachable state that is a deterministic function of the extended history, a response map depending only on that state, the declared input, and the current innovation, and a unifilar update from the same variables.  Two reachable extended histories at the same time are \emph{pathwise predictively equivalent} when every common future sequence of declared inputs and innovation values produces the same future response sequence.  A realization is \emph{behaviorally minimal} when it has no distinct reachable states with this equivalence.  The quotient below is taken in this reachable pathwise/unifilar category.  It is relative to the declared innovation representation and coupling.  A distribution-only minimal state modulo null sets is a different measurable-kernel problem and is not claimed here.

The Hilbert-valued response in \cref{eq:causal-response} and the scalar response used later have a precise interface.  For every continuous linear functional $\varphi\in\Hspace^*$ define $F_\varphi(a)=\varphi(F(a))$.  The smooth reduction theorem applies to $F_\varphi$ when this scalar observable has a regular reduced-value representation.  A vector response can be treated componentwise only when each selected component has such a representation.

\subsection{Pure interventional effects}

Let $\mu_{\Pcal}$ be the M\"obius function.  Define
\begin{equation}
\zeta_t=\sum_{s\preceq t}\mu_{\Pcal}(s,t)F(s),
\qquad
F(a)=\sum_{t\preceq a}\zeta_t.
\label{eq:mobius-response}
\end{equation}
For $\Pcal=2^{[m]}$,
\[
\zeta_T=\sum_{B\subseteq T}(-1)^{|T|-|B|}F(B),
\qquad
F(A)=\sum_{T\subseteq A}\zeta_T.
\]
The baseline is $\zeta_\varnothing$, singleton effects are main effects, and larger sets are irreducible factorial interactions under the fixed protocol.  This is a decomposition of counterfactual responses; it does not assert that the implementation internally adds vectors named $\zeta_T$.

A finite model declares an effect support $\Kcal\subseteq\Pcal$ and sets $\zeta_t=0$ outside $\Kcal$.  Main-effect, order-$q$, sparse hierarchical, and graded models are choices of $\Kcal$.  They are falsifiable because fitted effects predict held-out configurations.

\subsection{Observation and information}

For queried configurations $\Dcal=(a_1,\ldots,a_n)$, define
\[
X_{\Dcal,\Kcal}(i,t)=\mathbf 1\{t\preceq a_i\}.
\]
When $\Hspace\cong\R^d$, the observation equation is
\begin{equation}
Y=\Acal_{\Dcal}\zeta+\varepsilon,
\qquad
\Acal_{\Dcal}=X_{\Dcal,\Kcal}\otimes I_d,
\qquad
\mathcal I=\Acal_{\Dcal}^*\Sigma^{-1}\Acal_{\Dcal}.
\label{eq:observation-information}
\end{equation}
The kernel of $\Acal_{\Dcal}$ is the observational gauge.  Its positive singular spectrum controls stable recovery; the information operator controls finite-sample inference and design.

\subsection{Geometric response classes}

For a visible covector $\alpha_k$, let $Q_k=\eta_kP_k$ be the effective positive map and $u_k=-Q_k\alpha_k$.  Absorbing the step scale into $Q_k$ prevents temporal audits from hiding arbitrary variation in $\eta_k$.  Fix a finite-dimensional normed parameterization $\mathbb Q$ of the effective maps along the audited trace, using a declared trivialization when tangent spaces vary, and define
\[
\VarOp(Q_{0:N-1})
=\sum_{k=0}^{N-2}\norm{Q_{k+1}-Q_k}_{\mathbb Q}.
\]
For a family $\Fcal$ and path budget $B$, define
\begin{equation}
\Mcal_{\Fcal,B}
=\left\{(-Q_k\alpha_k)_{k=0}^{N-1}:
Q_k\in\Pcal_{\Fcal}(\theta_k),\
\VarOp(Q_{0:N-1})\le B\right\}\subseteq\Hspace.
\label{eq:restricted-response-class}
\end{equation}
Distance to this set tests membership in a declared response class.  Interventions identify counterfactual coordinates.  These operations share a response space and answer different questions.

%% file: sections/04_universal_representation.tex
\section{Universal Causal--Interventional Representation}
\label{sec:universal-representation}

\begin{theorem}[Universal causal--interventional optimizer representation]
\label{thm:universal-interventional-calculus}
Let $(\Pcal,\preceq)$ be lower-finite, let $F:\Pcal\to\Hspace$ be induced by \cref{eq:causal-response}, and let $\Kcal\subseteq\Pcal$ and $\Dcal$ be finite effect-support and design families.
\begin{enumerate}[leftmargin=*,itemsep=3pt,label=(\roman*)]
\item \emph{Minimal causal state.}  Every innovation-driven causal optimizer under the fixed protocol has a history-state realization.  Pathwise predictive-equivalence classes form a sufficient pathwise/unifilar realization.  Every other admissible sufficient reachable pathwise realization maps onto this quotient, and every behaviorally minimal reachable realization is isomorphic to it as a reachable transition--response system.
\item \emph{Incidence completeness.}  Equation~\eqref{eq:mobius-response} is the unique incidence-algebra expansion of $F$.
\item \emph{Complete experimental gauge.}  Two supported effect collections are observationally equivalent exactly when
\[
\zeta-\zeta'\in\ker(X_{\Dcal,\Kcal}\otimes I_{\Hspace}).
\]
If $\dim\Hspace=d$ and $r_{\Dcal}=\rank X_{\Dcal,\Kcal}$, every compatible fiber has dimension $d(|\Kcal|-r_{\Dcal})$.  Exact identification holds if and only if $r_{\Dcal}=|\Kcal|$.
\item \emph{Sharp quotient stability.}  If $X=X_{\Dcal,\Kcal}$ has positive rank and $y=(X\otimes I_d)\zeta+e$, the minimum-norm inverse estimates the canonical representative $\zeta_{\rm can}\in\ker(X\otimes I_d)^\perp$ with
\[
\norm{\widehat\zeta-\zeta_{\rm can}}
\le
\frac{\norm{\Proj_{\range(X\otimes I_d)}e}}{\sigma_{\min}^+(X)}.
\]
The constant is sharp.
\item \emph{Exact intervention complexity.}  Suppose $\Hspace\ne\{0\}$ and the supported class contains every collection in $\Hspace^{\Kcal}$.  In a noiseless expected-response oracle model, every fixed, deterministic adaptive, or randomized adaptive scheme that exactly recovers every supported law with probability one needs at least $|\Kcal|$ distinct configuration queries in the worst case.  Querying $\Dcal=\Kcal$ attains the bound:
\[
N_{\rm exact}(\Kcal)=|\Kcal|.
\]
\item \emph{Prediction and falsification.}  If $X_{\Dcal,\Kcal}$ has full column rank, the recovered effects predict
\[
F_{\Kcal}(b)=\sum_{t\in\Kcal:t\preceq b}\zeta_t
\qquad\text{for every }b\in\Pcal.
\]
A held-out response inconsistent with this equation rejects the support model under the fixed protocol.
\end{enumerate}
\end{theorem}

The proof is in \cref{app:theory-strengthening}.  The randomized lower bound uses the finite set of incidence row patterns induced by $\Kcal$; it does not require the ambient lower-finite poset itself to be finite.  The count $N_{\rm exact}$ is configuration complexity, not the number of noisy optimizer runs.  Replication under observation noise belongs to \cref{thm:projection-information-decision}.

\paragraph{Boolean and graded complexity.}
For $\Pcal=2^{[m]}$, unrestricted attribution costs $2^m$ configurations.  Under the order-$q$ support
\[
\Kcal_q=\{T\subseteq[m]:|T|\le q\},
\qquad
N_q=\sum_{j=0}^q\binom mj.
\]
Product-of-chain interventions yield graded mixed differences, with exact cost equal to the retained multi-index count.

\paragraph{Passive impossibility.}
If a finite poset has a greatest configuration $\top$, the one-row passive design $\Dcal=\{\top\}$ has rank one.  It leaves $d(|\Kcal|-1)$ unidentifiable coordinates.  A passive geometric residual can exclude a declared class, but universal causal attribution requires interventions or an equivalent injective observation.

\paragraph{No universal dissipative representation.}
The theorem includes causal rules with cycles.  For example, $x_{k+1}=x_k+1\pmod 3$ is causal but cannot satisfy a strict scalar Lyapunov law, since it would imply $V(0)>V(1)>V(2)>V(0)$.  Smooth gradient or Onsager dynamics are regular subclasses of the causal response theory.

%% file: sections/05_hidden_relaxation.tex
\section{Hidden Relaxation Generates Observable Interaction}
\label{sec:hidden-relaxation}

The universal theorem identifies effects without assigning them a smooth mechanism.  The companion variational-reduction paper P1 already proves the hidden-response, Schur-curvature, affine-intervention, factorial-integral, and noncommutation results used below \citep{li2026optimizationgeometrodynamics}.  We restate that structural input in the present notation, then derive a new readout-transfer law that applies to observed optimizer updates rather than only to the reduced optimal value.

Let $\mathcal Y$ and $\mathcal S$ be finite-dimensional smooth manifolds, let $y\in\mathcal Y$ be visible, $\sigma\in\mathcal S$ hidden, and let $u\in U\subset\R^m$ be a continuous intervention amplitude, where $U$ is open and contains $[0,1]^m$.  Derivatives below are coordinate matrices in a chosen local product chart; on manifolds they may instead be read as covariant derivatives for chosen connections.  The $u$ block is unambiguous in the declared Euclidean intervention coordinates.  Because $E_\sigma=0$ at the minimizing section, the Schur expression is invariant under smooth reparameterizations of the hidden coordinate.
Consider a scalar response generated by
\begin{equation}
\bar E(y,u)=\inf_{\sigma\in\mathcal S}E(y,\sigma,u).
\label{eq:reduced-intervention-energy}
\end{equation}

\begin{assumption}[Regular hidden relaxation]
\label{ass:regular-relaxation}
On $\mathcal Y\times U$, the energy is $C^{r+1}$ for some $r\ge2$ and is locally uniformly inf-compact in $\sigma$: over every compact parameter set, each finite fiber sublevel union is relatively compact.  Every $(y,u)\in\mathcal Y\times U$ has a unique interior minimizer $\sigma_\star(y,u)$, and
\[
H(y,u)=D^2_{\sigma\sigma}E(y,\sigma_\star(y,u),u)
\]
is positive definite.
\end{assumption}

\begin{theorem}[Reduction-induced optimizer interaction; P1, restated]
\label{thm:reduction-induced-interaction}
Under \cref{ass:regular-relaxation}, the minimizer section is $C^r$ and
\begin{align}
D_{(y,u)}\sigma_\star
&=-H^{-1}E_{\sigma(y,u)},
\label{eq:hidden-response-law}\\
D^2_{(y,u)(y,u)}\bar E
&=E_{(y,u)(y,u)}-E_{(y,u)\sigma}H^{-1}E_{\sigma(y,u)}.
\label{eq:effective-schur-hessian}
\end{align}
If interventions enter affinely before reduction,
\begin{equation}
E(y,\sigma,u)=E_0(y,\sigma)+\sum_{i=1}^m u_i\Delta_i(y,\sigma),
\label{eq:affine-interventions}
\end{equation}
and $G:\R^m\to T_\sigma^*\mathcal S$ is $Ga=\sum_i a_iD_\sigma\Delta_i$, then
\begin{equation}
\boxed{D^2_{uu}\bar E=-G^*H^{-1}G\preceq0.}
\label{eq:interaction-curvature}
\end{equation}
For Boolean vertex responses $F(A)=\bar E(y,\mathbf1_A)$ and every $T\subseteq[m]$ with $|T|\le r$,
\begin{equation}
\zeta_T
=\int_{[0,1]^T}
\partial_T^{|T|}\bar E
\left(y,\sum_{i\in T}t_ie_i\right)
\prod_{i\in T}\dd t_i.
\label{eq:mobius-derivative-integral}
\end{equation}
In particular,
\begin{equation}
\boxed{
\zeta_{\{i,j\}}
=-\int_0^1\!\int_0^1
\ip{D_\sigma\Delta_i}{H^{-1}D_\sigma\Delta_j}_{(se_i+te_j,\sigma_\star)}
\,\dd s\,\dd t.}
\label{eq:mobius-schur-bridge}
\end{equation}
\end{theorem}

Equation~\eqref{eq:interaction-curvature} is a negative-semidefinite Gram law.  A pointwise mixed response vanishes exactly when the corresponding vertical perturbations are orthogonal in the inverse-Hessian metric.  A zero integrated M\"obius effect may also result from cancellation, so the converse need not hold for \eqref{eq:mobius-schur-bridge}.

The bridge is a forward structural law.  A factorial experiment identifies its left-hand side, but the factorization into $H$ and $G$ is generally nonunique.  A nonzero pair effect therefore does not by itself identify hidden geometry, inverse stiffness, or a particular optimizer state.  Testing the right-hand side requires an independently specified and measured hidden model.

The reduced value is only one possible experimental response.  Actual optimizer traces usually observe a readout of the relaxed state.  The next result gives the missing interface.

\begin{theorem}[Observable readout interaction transfer]
\label{thm:readout-interaction-transfer}
Under \cref{ass:regular-relaxation}, fix $y$ and let $R:U\times\mathcal S\to\R$ be $C^2$.  Define the observed scalar response
\[
F_R(u)=R(u,\sigma_\star(u)).
\]
In the chosen hidden chart, put
\[
v_i=\partial_{u_i}\sigma_\star=-H^{-1}E_{\sigma u_i}
\]
and
\begin{equation}
w_{ij}=-H^{-1}\!\left(
E_{\sigma\sigma\sigma}[v_i,v_j]
+E_{\sigma\sigma u_i}v_j
+E_{\sigma\sigma u_j}v_i
+E_{\sigma u_i u_j}
\right).
\label{eq:second-hidden-response}
\end{equation}
Then $w_{ij}=\partial^2_{u_i u_j}\sigma_\star$ and
\begin{equation}
\boxed{
\partial^2_{u_i u_j}F_R
=R_{u_i u_j}+R_{u_i\sigma}v_j+R_{\sigma u_j}v_i
+R_{\sigma\sigma}[v_i,v_j]+R_\sigma w_{ij}.}
\label{eq:readout-interaction-transfer}
\end{equation}
Consequently the Boolean pair effect of the actual readout is
\begin{equation}
\zeta^R_{\{i,j\}}
=\int_0^1\!\int_0^1
\partial^2_{u_i u_j}F_R(se_i+te_j)\,\dd s\,\dd t.
\label{eq:readout-mobius-integral}
\end{equation}
For affine interventions, $E_{\sigma u_i u_j}=0$ and
$E_{\sigma\sigma u_i}=D^2_{\sigma\sigma}\Delta_i$.  Unlike the reduced-value law, a general readout need not have negative-semidefinite interaction curvature.  Hilbert-valued updates are covered after composition with any declared scalar functional $\varphi\in\Hspace^*$.
\end{theorem}

\begin{proposition}[Reduction and M\"obius transformation do not commute]
\label{prop:reduction-mobius-noncommutation}
There is no general identity between taking a hidden infimum before a M\"obius transform and taking an infimum of pointwise M\"obius differences.
\end{proposition}

Indeed, for $E_0(\sigma)=\sigma^2$ and $E_1(\sigma)=(\sigma-1)^2$, both reduced values are zero, so the reduced one-sector effect is zero.  The pointwise difference is $1-2\sigma$, whose infimum is $-\infty$.  Fiber-constant intervention increments are a sufficient commuting case: they leave $\sigma_\star$ unchanged and create no higher-order effects.

\paragraph{Exact analytical example.}
Let
\[
E(\sigma,u)=\frac h2\sigma^2+u_1a\sigma+u_2b\sigma,
\qquad h>0.
\]
Then $\sigma_\star=-(au_1+bu_2)/h$ and
\[
\bar E(u)=-\frac{(au_1+bu_2)^2}{2h},
\qquad
\zeta_{\{1,2\}}=-\frac{ab}{h}.
\]
The interaction is large when hidden relaxation is compliant ($h$ small), and its sign records alignment of the two vertical perturbations.

\paragraph{Coupled adaptive diagonal-preconditioner example.}
Let $g=(g_1,g_2)$ be the visible gradient and let $\sigma\in\R^2$ be a hidden log-second-moment state selecting
\[
P(\sigma)=\diag(e^{-\sigma_1/2},e^{-\sigma_2/2}),
\qquad q(\sigma,g)=-P(\sigma)g.
\]
For $h>|\rho|$, a statistic $b(g)\in\R^2$, and two declared protocol increments $\Delta_1(\sigma)=\sigma_1$, $\Delta_2(\sigma)=\sigma_2$, consider the regularized state-selection energy
\begin{equation}
E(g,\sigma,u)
=\frac12\sigma^\top
\underbrace{\begin{pmatrix}h&\rho\\ \rho&h\end{pmatrix}}_{H}
\sigma-b(g)^\top\sigma+u_1\Delta_1(\sigma)+u_2\Delta_2(\sigma).
\label{eq:coupled-log-preconditioner}
\end{equation}
Here $G=[e_1\ e_2]=I_2$, $\sigma_\star=H^{-1}(b(g)-u)$, and the scalar observable is the optimized calibration value $F(A)=\bar E(g,\mathbf1_A)$.  The Boolean pair effect is
\begin{equation}
\zeta_{\{1,2\}}
=-e_1^\top H^{-1}e_2
=\frac{\rho}{h^2-\rho^2}.
\label{eq:coupled-log-preconditioner-pair}
\end{equation}
Thus interventions acting on different diagonal channels interact whenever the hidden adaptation penalty couples those channels.  More directly, the first coordinate of the actual optimizer update is a readout
\[
q_1(u)=-g_1e^{-\sigma_{\star,1}(u)/2}
=q_1(0)\exp\!\left(
\frac{h}{2(h^2-\rho^2)}u_1
-\frac{\rho}{2(h^2-\rho^2)}u_2\right).
\]
Its identifiable Boolean pair effect is
\begin{equation}
\zeta^{q_1}_{\{1,2\}}
=q_1(0)
\left(e^{\frac{h}{2(h^2-\rho^2)}}-1\right)
\left(e^{-\frac{\rho}{2(h^2-\rho^2)}}-1\right),
\label{eq:coupled-update-readout-pair}
\end{equation}
which is nonzero when $g_1\rho\ne0$.  The calibration-value contrast in \cref{eq:coupled-log-preconditioner-pair} follows the P1 Gram law; the observed-update contrast follows the new readout law and has no universal sign.  This example specifies the optimizer state, selected positive geometry, energy, intervention increments, vertical Hessian, perturbation map, reduced response, actual update response, and both predicted factorial contrasts.  Its full calculation is in \cref{app:interaction-proof}.

\paragraph{Structured regular subclasses.}
For information-induced SPD completion, $H^{-1}$ is the inverse vertical curvature of the completion energy supplied by P4 \citep{li2026informationinducedgeometry}.  P1 supplies the finite-dimensional reduction and reduced-value interaction laws restated above \citep{li2026optimizationgeometrodynamics}.  P2 supplies the path-space Jacobi--Schur and M\"obius--Jacobi versions \citep{li2026restrictedcomplexity}.  The new readout theorem transports either hidden response to experimentally chosen update or trace observables.  These realizations require their own regularity assumptions and do not restrict \cref{thm:universal-interventional-calculus}.

%% file: sections/06_observation_design.tex
\section{Observation, Inference, and Experimental Design}
\label{sec:observation-design}

Let $\mathbb Y\cong\R^n$ have covariance $V\succ0$ and norm $\norm y_{V^{-1}}^2=y^\top V^{-1}y$.  Let $\Mcal\subseteq\mathbb Y$ be a nonempty closed convex response class, let $\mathbb Z$ be a finite-dimensional real inner-product space, and let the linear map $A:\mathbb Z\to\mathbb Y$ generate responses from coordinates $z\in\mathbb Z$.  Write
\[
\mathcal I=A^*V^{-1}A,
\qquad K=\ker A,
\qquad r=\rank A.
\]

\begin{theorem}[Identifiable observation--inference--decision calculus]
\label{thm:projection-information-decision}
Under this setup:
\begin{enumerate}[leftmargin=*,itemsep=4pt,label=(\roman*)]
\item Every $y$ has a unique shadow $\widehat m=\argminop_{m\in\Mcal}\frac12\norm{y-m}_{V^{-1}}^2$.  Define $\delta_{\Mcal}(y)=\norm{y-\widehat m}_{V^{-1}}$.  With the extended-real support function $\sigma_{\Mcal}(q)=\sup_{m\in\Mcal}q^\top m\in\R\cup\{+\infty\}$,
\begin{equation}
\frac12\delta_{\Mcal}(y)^2
=\max_q\left\{q^\top y-\sigma_{\Mcal}(q)-\frac12q^\top Vq\right\},
\label{eq:master-projection-dual}
\end{equation}
and $q^\star=V^{-1}(y-\widehat m)$.  Distance is one-Lipschitz and decreases under nesting of response classes.
\item In $Y=Az+\varepsilon$, $\varepsilon\sim N(0,V)$, the identifiable parameter is the quotient by $K$.  On $K^\perp$, generalized least squares
\[
\widehat z=\mathcal I^\dagger A^*V^{-1}Y
\]
is exactly minimax:
\begin{equation}
\inf_{\widetilde z}\sup_{z\in K^\perp}\EE_z\norm{\widetilde z-z}^2
=\tr(\mathcal I^\dagger),
\label{eq:master-minimax-risk}
\end{equation}
and $(\widehat z-z)^*\mathcal I(\widehat z-z)\sim\chi_r^2$.
\item Partition $A=[A_{-G}\ A_G]$, put $\widetilde A=V^{-1/2}A$, let $\Pi_{-G}$ project onto $\range(\widetilde A_{-G})$, and define $B_G=(I-\Pi_{-G})\widetilde A_G$.  The statistic
\[
\norm{\Proj_{\range(B_G)}V^{-1/2}Y}^2
\]
is central $\chi^2_{\rank B_G}$ under $B_Gz_G=0$ and noncentral with parameter $\norm{B_Gz_G}^2$ otherwise.  If the true mean is $Az+b$ with the canonical representative $z\in K^\perp$, let $\Pi_A$ project onto $\range(V^{-1/2}A)$ and define
\[
b_\parallel=\mathcal I^\dagger A^*V^{-1}b,
\qquad
b_\perp=(I-\Pi_A)V^{-1/2}b.
\]
Omitted structure decomposes exactly into estimation bias and orthogonal lack of fit:
\begin{equation}
\EE\norm{\widehat z-z}^2=\norm{b_\parallel}^2+\tr(\mathcal I^\dagger),
\qquad
\norm{(I-\Pi_A)V^{-1/2}Y}^2\sim\chi^2_{n-r}(\norm{b_\perp}^2).
\label{eq:master-misspecification}
\end{equation}
\item Every linear downstream map obeys
\[
\norm{L(y-\widehat m)}\le\opnorm{LV^{1/2}}\delta_{\Mcal}(y).
\]
For a declared finite action class, or more generally an action class on which the displayed minima are attained, let $\Delta(\pi;z)=d_\pi+\ell_\pi^\top z$ with $\ell_\pi\in K^\perp$ and define
\[
\rho_\pi(\delta)=\sqrt{\chi^2_{r,1-\delta}}
\sqrt{\ell_\pi^\top\mathcal I^\dagger\ell_\pi},
\qquad
\widehat\pi\in\argminop_\pi
\{\Delta(\pi;\widehat z)+\rho_\pi(\delta)\}.
\]
On one event of probability $1-\delta$,
\begin{equation}
\Delta(\widehat\pi;z)
\le\Delta(\pi^\star;z)+2\rho_{\pi^\star}(\delta),
\qquad
\pi^\star\in\argminop_\pi\Delta(\pi;z).
\label{eq:master-action-oracle}
\end{equation}
\item For a saturated incidence design $\Dcal=\Kcal$, let $M=X_{\Kcal,\Kcal}^{-1}$.  With independent response covariance $\Sigma$ and continuous design weights $n_a>0$, $\sum_an_a=N$, the exact relaxed total recovery risk is
\begin{equation}
R(n)=\tr(\Sigma)\sum_a\frac{c_a}{n_a},
\qquad c_a=\norm{M_{:,a}}_2^2,
\label{eq:exact-replication-risk}
\end{equation}
with unique optimum
\begin{equation}
n_a^\star=N\frac{\sqrt{c_a}}{\sum_b\sqrt{c_b}},
\qquad
R^\star=\frac{\tr(\Sigma)}{N}\left(\sum_a\sqrt{c_a}\right)^2.
\label{eq:exact-a-optimal-allocation}
\end{equation}
For the unrestricted Boolean lattice,
\begin{equation}
c_A=2^{m-|A|},
\quad
n_A^\star=N\frac{2^{(m-|A|)/2}}{(1+\sqrt2)^m},
\quad
R^\star=\frac{\tr(\Sigma)}{N}(3+2\sqrt2)^m.
\label{eq:boolean-exact-a-optimal-allocation}
\end{equation}
These formulas solve the continuous allocation problem exactly.  Integer replication counts require a feasible rounding rule and its associated excess-risk analysis.
\end{enumerate}
\end{theorem}

\begin{corollary}[Exact target-optimal replication]
\label{cor:target-optimal-replication}
In the saturated scalar incidence design, let $M=X_{\Kcal,\Kcal}^{-1}$ and
let $L$ select or linearly combine target effects.  Suppose the independent
configuration-level observation variance is $\sigma_a^2/n_a$ and
\[
c_a=\norm{LM_{:,a}}_2^2>0.
\]
Then the exact target risk and its continuous optimum are
\begin{equation}
R_L(n)=\tr\!\left[L M\diag\!\left(\frac{\sigma_a^2}{n_a}\right)M^\top L^\top\right]
=\sum_a\frac{c_a\sigma_a^2}{n_a},
\label{eq:target-replication-risk}
\end{equation}
\begin{equation}
n_a^\star
=N\frac{\sigma_a\sqrt{c_a}}{\sum_b\sigma_b\sqrt{c_b}},
\qquad
R_L^\star
=\frac1N\left(\sum_a\sigma_a\sqrt{c_a}\right)^2.
\label{eq:target-optimal-allocation}
\end{equation}
The integer design used in \cref{sec:experiments} rounds these weights while
preserving the total budget and a declared minimum replication.
\end{corollary}

The proof is in \cref{app:deep-theory}.  The theorem turns the deterministic gauge of \cref{thm:universal-interventional-calculus} into finite-sample uncertainty.  Rank determines what is identifiable; the positive spectrum determines how accurately; the target decision determines which directions matter.

\paragraph{Exact non-Gaussian calibration.}
If a finite group $G$ leaves the null response law invariant, then for any residual statistic $T$,
\begin{equation}
p_G(Y)=\frac1{|G|}\sum_{g\in G}\mathbf1\{T(gY)\ge T(Y)\}
\label{eq:exact-randomization-pvalue}
\end{equation}
is super-uniform under the null \citep{hemerik2018exact,phipson2010permutation}.  This calibrates arbitrary response-class residuals when the group action is scientifically justified.

\paragraph{Design for interaction curvature.}
For the mask experiment $A=X_{\Dcal,\Kcal}\otimes I_d$, design can maximize rank, the smallest positive information eigenvalue, or a target contrast precision.  To identify the left-hand side of \cref{eq:mobius-schur-bridge}, the support must include the relevant pair effect and the queried masks must make its residualized column nonzero.  An order-two design with that rank property is the minimal direct identification extension; separate held-out higher masks test the order-two support assumption.  Validation of the hidden right-hand side additionally needs measurements or a separately fitted structural model for $H$ and $G$.

%% file: sections/07_geometric_response_classes.tex
\section{Structured Geometry as a Response Class}
\label{sec:geometric-response}

Geometry is one declared sector inside the causal calculus.  This section records the response classes used by the experiments and clarifies their representational boundaries.

\subsection{Pointwise expressivity}

\begin{proposition}[Full-SPD geometricization]
\label{prop:full-spd-geometricization}
For nonzero $g\in\R^d$ and $u\in\R^d$, there exists $P\in\SPD^d$ with $u=-Pg$ if and only if $g^\top u<0$.
\end{proposition}

Thus full positive geometry expresses exactly strict descent directions relative to the declared visible covector.  At an exact critical point, every pure positive-geometric update is zero.

\begin{corollary}[Critical-point boundary]
\label{cor:critical-boundary}
If $g=0$ and $u\ne0$, the motion requires another declared sector, an enlarged state-space geometry, or a non-gradient mechanism.
\end{corollary}

For the positive diagonal family $P=\diag(p_1,\ldots,p_d)$, $p_i>0$:

\begin{proposition}[Diagonal expressivity and residual]
\label{prop:diagonal-expressivity}
Exact expression holds if and only if $g_i=0\Rightarrow u_i=0$ and $g_i\ne0\Rightarrow u_ig_i<0$ for every $i$.  Moreover,
\[
\rho_{\rm diag}(u;g)^2
=\sum_{i:g_i=0}u_i^2+
\sum_{i:g_i\ne0,\,u_ig_i>0}u_i^2.
\]
If $g_i\ne0$ and $u_i=0$, the coordinate infimum is zero but is not attained in the open positive cone.
\end{proposition}

For a coordinate partition $\{B_1,\ldots,B_s\}$:

\begin{proposition}[Block expressivity]
\label{prop:block-expressivity}
A block-diagonal $P\succ0$ satisfies $u=-Pg$ if and only if each block obeys $g_{B_j}=0\Rightarrow u_{B_j}=0$ and $g_{B_j}\ne0\Rightarrow g_{B_j}^\top u_{B_j}<0$.
\end{proposition}

The proofs and the closed-form bounded-diagonal projection are in \cref{app:proofs}.  Diagonal and block statements depend on parameterization; intrinsic claims require an intrinsic reference geometry and family.

\subsection{Trajectory residual complexity}

For a trace $\Tcal_N=\{(\theta_k,\alpha_k,u_k)\}_{k=0}^{N-1}$ with $\norm u>0$, stack the updates in the product response space and define
\begin{equation}
\TRC_{\Fcal,B}(\Tcal_N)
=\inf_{m\in\Mcal_{\Fcal,B}}\norm{u-m},
\qquad
\TGER_{\Fcal,B}=1-
\frac{\TRC_{\Fcal,B}}{\norm u}.
\label{eq:trajectory-residual-complexity}
\end{equation}
The raw residual permits a fresh map at each step and therefore audits a high-capacity response class.  A coherent budget restricts temporal variation; $B=0$ fits one shared map in the declared parameterization.  Since the model class need not contain the zero response, $\TGER\le1$ may be negative; zero means that the best class member is no closer than the zero response under the chosen norm.

Closedness and convexity require structure.  A sufficient condition for the audited finite traces is that every parameter fiber is nonempty compact convex in the common finite-dimensional space $\mathbb Q$, at least one parameter path satisfies all fiber constraints and the variation budget, $Q\mapsto-Q\alpha_k$ is linear, and $\VarOp$ is the convex continuous variation defined in \cref{sec:language}.  The feasible parameter path set is then nonempty compact convex and its response image is nonempty compact convex.  This covers the bounded-diagonal, bounded block, channel-scalar, and layer-scalar audits, whose shared constant maps provide budget-feasible paths.  Open positive cones and arbitrary nonlinear SPD path families do not inherit this conclusion.  Under the sufficient condition, \cref{thm:projection-information-decision} gives a unique shadow, a dual residual certificate, nesting monotonicity, and perturbation stability.

For the accumulation map $L_t(r_0,\ldots,r_{N-1})=\sum_{k<t}r_k$,
\[
\norm{\theta_t-\bar\theta_t}
\le\opnorm{L_t}\TRC_{\Fcal,B}.
\]
Thus a response-space residual transfers to trajectory shadowing.  It does not identify which implementation module caused the residual.

\subsection{Relation to the companion geometry papers}

P4 studies when a visible SPD action has a unique minimal full completion and supplies exact AIRM quotient geometry \citep{li2026informationinducedgeometry}.  P1 studies smooth hidden-variable reduction and metric evolution \citep{li2026optimizationgeometrodynamics}.  P2 restricts geometry families and path budgets \citep{li2026restrictedcomplexity}.  Here their outputs become response classes and intervention coordinates.  The causal and statistical theorems remain valid for nongeometric sectors as well.

%% file: sections/08_experiments.tex
\section{Factorial Closure and Response-Class Audits}
\label{sec:experiments}

The evidence has two deliberately different layers.  A controlled real-data
experiment closes the complete chain from hidden strong convexity through
Schur curvature, M\"obius effects, experimental design, inference,
falsification, and finite-response tracking.  Neural-network trace audits then
test whether the declared response classes remain informative in nonconvex
training.  Only the first layer is a direct validation of the
reduced-value interaction law.

\subsection{A real-data hidden optimization model}

We use the handwritten digits 3 and 8 from the scikit-learn digits data.  A
stratified split gives 249 training and 108 test examples.  After standardizing
the 64 pixels and adding an intercept, the hidden state is
$w\in\R^{65}$.  The base energy is ridge logistic regression,
\[
E_0(w)=\frac1N\sum_{n=1}^N
\log(1+e^{-y_nx_n^\top w})+\frac{\lambda}{2}\norm w^2,
\qquad \lambda=0.15.
\]
Three continuous interventions $u\in[0,1]^3$ emphasize the digit-3 examples,
the digit-8 examples, and a central-$2\times2$-occlusion robustness loss.  If
$\Delta_i$ is the corresponding mean logistic loss, the hidden energy and
declared response are
\begin{equation}
E(w,u)=E_0(w)+0.75\sum_{i=1}^3u_i\Delta_i(w),
\qquad
F(u)=\min_w E(w,u).
\label{eq:digits-hidden-energy}
\end{equation}
Every added loss is convex, so $E_{ww}\succeq0.15I$ on the whole intervention
cube and every configuration has a unique global minimizer.  The largest
terminal gradient infinity norm across the eight Boolean solves is
$3.09\times10^{-13}$.

\subsection{Curvature, Boolean effects, and continuous prediction}

Boolean effects are computed from eight independently optimized corner
responses.  Separately, order-24 Gauss--Legendre quadrature evaluates
$-g_i^\top H^{-1}g_j$ after resolving the hidden optimum at every quadrature
node.  Neither path uses the output of the other.

\begin{table}[t]
\centering
\caption{Independent Boolean and curvature-integral evaluations for the three
mechanism pairs in \cref{eq:digits-hidden-energy}.  Pair 1--2 emphasizes the
two classes; pairs 1--3 and 2--3 combine class information with occlusion
robustness.}
\label{tab:interaction-curvature-closure}
\small
\input{figures/interaction_curvature_table.tex}
\end{table}

\begin{figure}[t]
\centering
\includegraphics[width=0.72\linewidth]{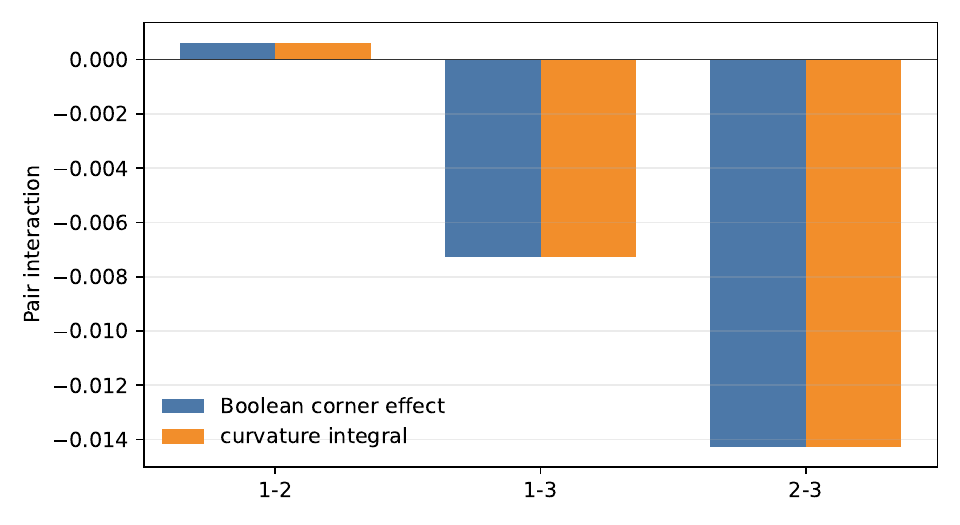}
\caption{Boolean pair effects and independently integrated Schur curvature.
The largest absolute discrepancy is $4.21\times10^{-11}$.}
\label{fig:interaction-curvature-validation}
\end{figure}

Across the three pairs, the maximum Boolean identity error is
$4.21\times10^{-11}$.  We also reserve three non-Boolean intensity points per
pair, $(0.2,0.7)$, $(0.45,0.35)$, and $(0.8,0.55)$.  Direct joint optimization
and edge-response-plus-curvature reconstruction agree at all nine points with
maximum error $8.88\times10^{-13}$.  Repeating all Boolean solves and three
curvature integrals on eight independent stratified data splits gives 24
additional identities with maximum error $2.81\times10^{-8}$.  The small
class--class effect changes sign across splits, whereas both
class--occlusion effects remain negative.

\subsection{Information design and Gaussian confirmation}

We declare the order-two support
\[
\Kcal_2=\{\varnothing,1,2,3,12,13,23\}
\]
and query its seven masks.  Their incidence matrix is invertible.  Mask $123$
is held out to test the omitted third-order effect, whose exact value is
\[
\zeta_{123}=0.002366036.
\]
By \cref{cor:target-optimal-replication}, the target risk for the three pair
effects has allocation weights $\sigma_a\norm{LM_{:,a}}$.  For the declared independent Gaussian channel
$Y_{a,r}=F(a)+\varepsilon_{a,r}$ with
$\sigma_a=0.004(1+0.2|a|)$, the pilot budget $(420,80)$ for main and held-out
replications has only $0.240$ exact power.  Scaling both budgets by the first
integer that reaches the predeclared $0.95$ target gives 3780 main and 720
held-out replications.  Pair-A-optimal integer allocation is
\[
(594,582,582,582,480,480,480).
\]

In 100,000 independent confirmatory campaigns, exact theoretical power is
$0.96309$ and empirical rejection is $0.96260$.  Held-out null rejection is
$0.04958$, and the joint 95\% pair ellipsoid covers with probability
$0.94973$.  Pair-A-optimal theoretical and empirical risks are respectively
$5.1433\times10^{-7}$ and $5.1297\times10^{-7}$, both below the uniform-design
values $5.1911\times10^{-7}$ and $5.2050\times10^{-7}$.  The improvement is
only $0.92\%$ under this mild heteroscedasticity; the point is agreement of
the information calculation and the Monte Carlo experiment, not a large
practical gain.

\subsection{A data-driven minibatch channel}

The Gaussian channel tests the exact finite-sample theorem under its stated
model.  A second channel uses only real minibatch losses.  At each optimized
configuration, independently sampled minibatches provide an unbiased estimate
of every active loss component in \cref{eq:digits-hidden-energy}.  A pilot of
200 observations per mask estimates heteroscedastic standard deviations and
sets, before confirmation, the allocation
\[
(624,537,614,625,386,458,536),
\]
again with 3780 main and 720 held-out observations.  All 4500 confirmatory raw
responses are retained.  Stratified nonparametric bootstrap with 20,000 draws
gives
\[
\begin{array}{ccl}
\zeta_{12}:&-0.000251,&[-0.001094,\ 0.000601],\\
\zeta_{13}:&-0.008050,&[-0.008900,-0.007226],\\
\zeta_{23}:&-0.014293,&[-0.015164,-0.013418].
\end{array}
\]
The two stable class--occlusion intervals contain their fixed-split truths and
exclude zero.  The small class--class interval includes zero and misses its
fixed-split truth $0.000625$ by $2.35\times10^{-5}$; we retain this miss rather
than claiming simultaneous coverage from three marginal 95\% intervals.  The
independent held-out triple estimate is $0.003114$ with interval
$[0.001839,0.004399]$, rejecting the order-two support under real,
heteroscedastic, non-Gaussian sampling noise.  Bootstrap pair risk
$5.7044\times10^{-7}$ agrees with the information estimate
$5.7297\times10^{-7}$ and is below the uniform estimate
$5.8620\times10^{-7}$.

\subsection{Finite-response tracking}

The Schur law is quasistatic, whereas an optimizer reaches $w_\star(u)$ in
finite time.  We run fixed-step gradient descent at the analytic smoothness
bound.  Strong convexity supplies a geometric objective-gap bound, and the sum
of four corner gaps bounds every pair-effect error.  At every audited horizon
the actual error is below the exact gap bound, which is in turn below the
strong-convexity bound.

\begin{figure}[t]
\centering
\includegraphics[width=0.72\linewidth]{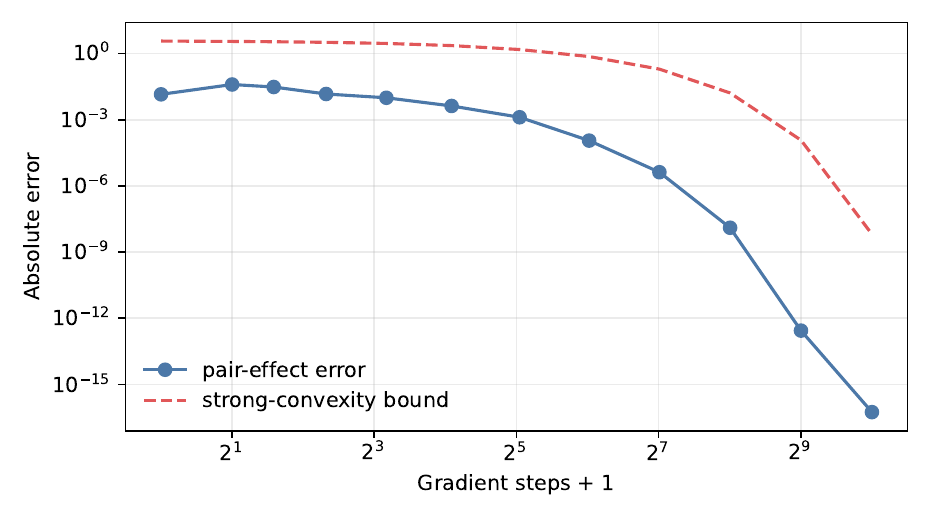}
\caption{Convergence of finite-horizon pair effects to the quasistatic effects.
The maximum error decreases from $1.43\times10^{-2}$ at initialization to
$5.55\times10^{-17}$ after 1024 steps and stays below the theoretical bound.}
\label{fig:interaction-finite-response}
\end{figure}

\subsection{Neural response-class audits}

The controlled experiment validates the theorem under global strong
convexity.  We retain the earlier neural audits for a different purpose:
testing whether declared geometric response classes distinguish optimizer
traces in nonconvex training.  A synthetic sanity check first separates fixed
diagonal, time-varying diagonal, momentum, and operator-mixing mechanisms.

\begin{table}[t]
\centering
\caption{Synthetic response-class sanity check.  Raw TGER permits a
step-specific diagonal map; coherent TGER fits one map over the trace.}
\label{tab:followup-synthetic-sanity}
\small
\input{figures/paper03_followup_synthetic_sanity_table.tex}
\end{table}

The long audits record 4096 pre-step gradients and updates for three seeds on
MNIST/Fashion-MNIST and four on CIFAR-10.  Bounded diagonal entries lie in
$[10^{-6},1]$; layer scalar uses one positive scalar per tensor.

\begin{table}[t]
\centering
\caption{Trace-audit seed means.  BDiag uses bounded-diagonal raw TGER under
the raw-gradient convention; +WD changes the visible covector; Layer uses the
layer-scalar family.  Sample standard deviations are retained in the supplied
CSV files.}
\label{tab:rich-convention-summary}
\small
\input{figures/paper03_rich_convention_table.tex}
\end{table}

\begin{figure}[t]
\centering
\includegraphics[width=0.72\linewidth]{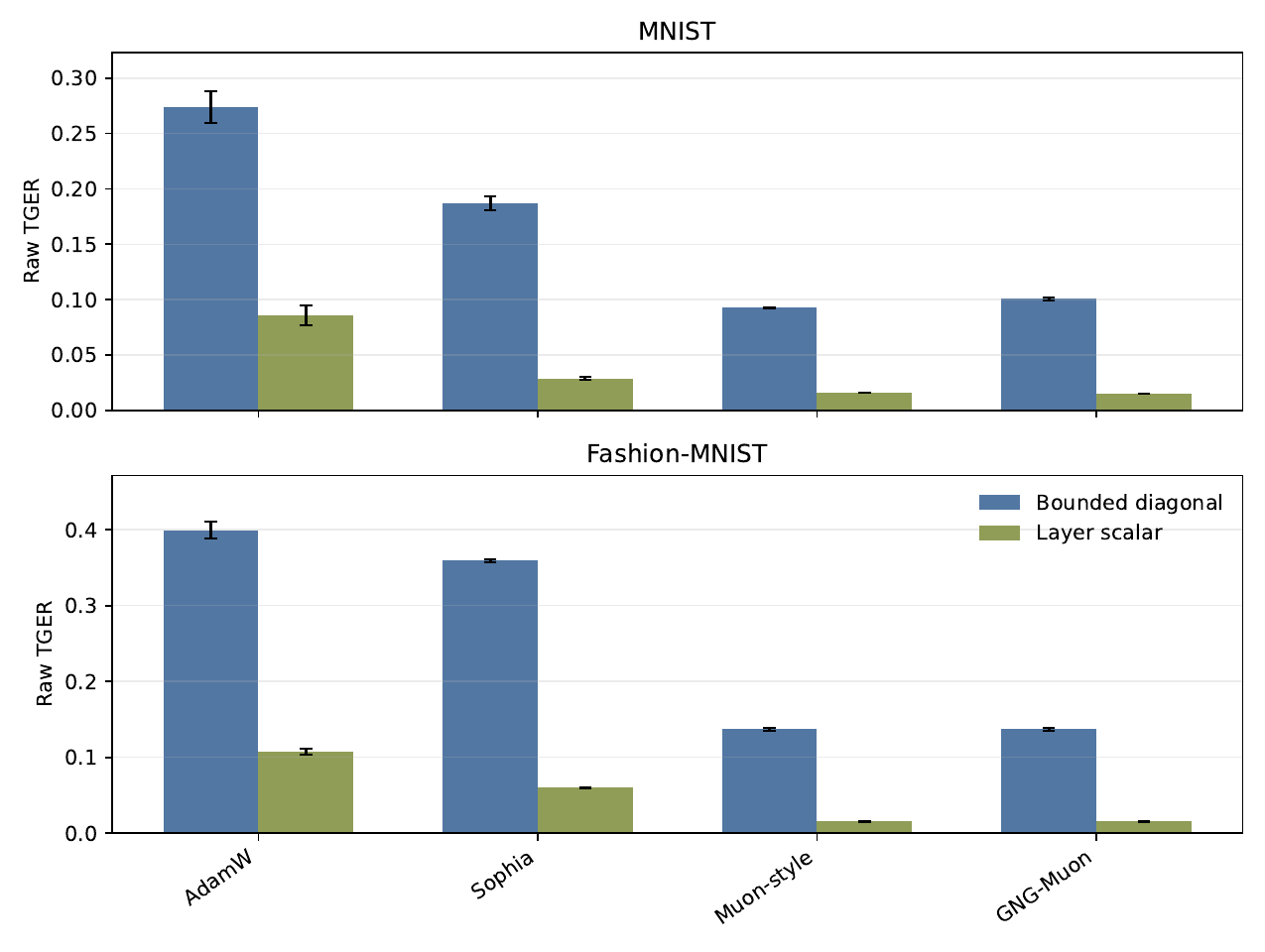}
\caption{Response-class sensitivity.  Layer scalar is nested inside bounded
diagonal geometry, and the explanation-rate drop varies by optimizer.}
\label{fig:family-sensitivity}
\end{figure}

A follow-up campaign adds channel scalar, an Adafactor-style row, paired
protocol changes, and a tiny-ResNet.  The intervention table reports
ablated-minus-base paired means and sample standard deviations across three
seeds.

\begin{table}[t]
\centering
\caption{Paired protocol interventions.  The SGD contrast jointly removes
momentum and Nesterov structure; the Lion contrast removes the sign operator
and changes update scale at fixed learning rate.}
\label{tab:followup-interventions}
\small
\input{figures/paper03_followup_intervention_table.tex}
\end{table}

\begin{table}[t]
\centering
\caption{CIFAR-10 small-ConvNet geometry ladder.  The three columns are nested
response families, not optimizer rankings.}
\label{tab:followup-sector-geometry}
\small
\input{figures/paper03_followup_sector_geometry_table.tex}
\end{table}

\begin{table}[t]
\centering
\caption{CIFAR-10 tiny-ResNet geometry ladder under the same short-horizon
response audit.}
\label{tab:followup-tinyresnet-geometry}
\small
\input{figures/paper03_followup_tinyresnet_geometry_table.tex}
\end{table}

The bundled SGD/Nesterov and Lion sign/scale interventions change both
accuracy and geometric residuals; AdamW weight-decay removal has little effect
over the short trace.  These pairs estimate bundled contrasts and do not
isolate momentum from Nesterov lookahead or sign from update scale.  They also
lack factorial masks and hidden-state measurements.  Accordingly, they support
response-class relevance in nonconvex training, while the digits experiment
provides the direct causal-geometric closure.  The current experiments validate
the P1 reduced-value law and the P3 identification/inference layer; the more
general observable-readout law in \cref{thm:readout-interaction-transfer}
remains an exact analytical prediction for a future nonconvex factorial
campaign.

%% file: figures/interaction_curvature_table.tex
\begin{tabular}{lrrr}
\toprule
Pair & Boolean effect & Curvature integral & Abs. error \\
\midrule
1-2 & +0.000625 & +0.000625 & 8.36e-15 \\
1-3 & -0.007276 & -0.007276 & 4.21e-11 \\
2-3 & -0.014264 & -0.014264 & 3.16e-11 \\
\bottomrule
\end{tabular}

%% file: figures/paper03_followup_synthetic_sanity_table.tex
\begin{tabular}{lrrrr}
\toprule
Mechanism & Mean GER & Raw TGER & Coh. TGER & Gap \\
\midrule
constant diagonal & 1.000 & 1.000 & 1.000 & 0.000 \\
varying diagonal & 1.000 & 1.000 & 0.520 & 0.480 \\
momentum memory & 0.484 & 0.455 & 0.227 & 0.227 \\
decoupled decay & 0.925 & 0.924 & 0.761 & 0.163 \\
operator mixing & 0.785 & 0.785 & 0.494 & 0.291 \\
\bottomrule
\end{tabular}

%% file: figures/paper03_rich_convention_table.tex
\begin{tabular}{llrrrr}
\toprule
Dataset & Opt. & Acc. (\%) & BDiag & +WD covec. & Layer \\
\midrule
MNIST & AdamW & 97.6 & 0.274 & 0.256 & 0.086 \\
 & Sophia & 97.5 & 0.187 & 0.278 & 0.029 \\
 & Muon-style & 97.6 & 0.093 & 0.228 & 0.016 \\
 & GNG-Muon & 97.9 & 0.101 & 0.232 & 0.015 \\
\midrule
Fashion-MNIST & AdamW & 88.6 & 0.400 & 0.364 & 0.108 \\
 & Sophia & 88.1 & 0.359 & 0.350 & 0.060 \\
 & Muon-style & 86.9 & 0.137 & 0.241 & 0.016 \\
 & GNG-Muon & 87.3 & 0.137 & 0.241 & 0.015 \\
\midrule
CIFAR-10 & AdamW & 74.6 & 0.367 & 0.338 & 0.088 \\
 & SGD/Nest. & 71.5 & 0.896 & 0.900 & 0.679 \\
 & Lion & 74.2 & 0.400 & 0.336 & 0.188 \\
 & Sophia & 74.7 & 0.345 & 0.335 & 0.035 \\
 & Muon-style & 74.8 & 0.228 & 0.263 & 0.018 \\
 & GNG-Muon & 73.2 & 0.199 & 0.252 & 0.020 \\
\bottomrule
\end{tabular}

%% file: figures/paper03_followup_intervention_table.tex
\begin{tabular}{llrrrr}
\toprule
Intervention & Pair & $\Delta$Acc. & $\Delta$BDiag & $\Delta$Channel & $\Delta$Layer \\
\midrule
weight decay & AdamW $\to$ AdamW/no WD & -0.2$\pm$0.1 & +0.000$\pm$0.001 & -0.000$\pm$0.001 & +0.000$\pm$0.000 \\
momentum/Nesterov & SGD/Nest. $\to$ SGD/no mom. & -18.9$\pm$0.1 & +0.131$\pm$0.009 & +0.218$\pm$0.004 & +0.255$\pm$0.005 \\
sign/scale & Lion $\to$ Lion/no sign & -39.5$\pm$1.9 & +0.191$\pm$0.049 & +0.052$\pm$0.036 & +0.071$\pm$0.036 \\
\bottomrule
\end{tabular}

%% file: figures/paper03_followup_sector_geometry_table.tex
\begin{tabular}{lrrrr}
\toprule
Opt. & Acc. (\%) & BDiag & Channel & Layer \\
\midrule
AdamW & 61.6 & 0.197 & 0.109 & 0.072 \\
AdamW/no WD & 61.4 & 0.198 & 0.109 & 0.072 \\
SGD/Nest. & 48.7 & 0.750 & 0.510 & 0.471 \\
SGD/no mom. & 29.8 & 0.881 & 0.728 & 0.727 \\
Lion & 50.3 & 0.412 & 0.241 & 0.197 \\
Lion/no sign & 10.8 & 0.603 & 0.293 & 0.268 \\
Sophia & 45.8 & 0.284 & 0.072 & 0.041 \\
Adafactor-style & 46.2 & 0.179 & 0.173 & 0.162 \\
Muon-style & 72.8 & 0.255 & 0.028 & 0.015 \\
GNG-Muon & 72.3 & 0.229 & 0.025 & 0.013 \\
\bottomrule
\end{tabular}

%% file: figures/paper03_followup_tinyresnet_geometry_table.tex
\begin{tabular}{lrrrr}
\toprule
Opt. & Acc. (\%) & BDiag & Channel & Layer \\
\midrule
AdamW & 58.6 & 0.300 & 0.159 & 0.124 \\
SGD/Nest. & 46.7 & 0.515 & 0.342 & 0.273 \\
Lion & 47.9 & 0.405 & 0.173 & 0.145 \\
Sophia & 39.5 & 0.392 & 0.099 & 0.071 \\
Adafactor-style & 40.7 & 0.217 & 0.216 & 0.210 \\
Muon-style & 61.0 & 0.173 & 0.043 & 0.027 \\
\bottomrule
\end{tabular}

%% file: sections/09_discussion.tex
\section{Discussion, Boundaries, and Conclusion}
\label{sec:discussion}

\paragraph{What is complete.}
Pathwise history realization is complete for finite-horizon innovation-driven laws under the fixed coupling, M\"obius inversion is complete for lower-finite intervention responses, and the observation kernel is the complete indistinguishability gauge for a declared finite support.  Completeness concerns representation and identification under a fixed protocol; it does not provide a distribution-only minimal stochastic state or a universal predictive transition map across protocols.

\paragraph{What the interaction layer adds.}
P1 explains reduced-value interaction through inverse vertical stiffness, and P2 gives the path-space version.  The new readout theorem shows how that hidden response reaches an experimentally selected update, trace, or task observable.  Its extra readout and second-response terms explain why an actual optimizer response need not inherit the P1 negative-semidefinite sign.  The same observed pair effect can admit multiple hidden models, so the calculus does not identify geometry from responses alone.  Nonunique, nonsmooth, or hysteretic hidden responses require generalized derivatives or set-valued effects.

\paragraph{Projection and causality.}
A small residual proves proximity to a response class.  It does not allocate causal credit.  A passive top configuration leaves $d(|\Kcal|-1)$ supported coordinates ambiguous.  Controlled configurations, additional instrumentation, or structural restrictions are required for attribution.

\paragraph{The price of exact structure.}
Unrestricted Boolean attribution requires $2^m$ configurations, and its saturated Gaussian risk grows as $(3+2\sqrt2)^m/N$.  Scientific tractability comes from falsifiable structure: low interaction order, sparse hierarchy, graded interventions, or a target decision subspace.  Adaptive support selection requires selective-inference or multiplicity control beyond the fixed-support theorem.

\paragraph{Statistical boundary.}
Exact minimax, confidence, and chi-square conclusions assume a finite-dimensional linear experiment with known positive-definite covariance.  Estimated covariance, dependent traces, and adaptive designs need additional analysis.  Randomization calibration is exact only under a defensible invariant group action.

\paragraph{Empirical boundary.}
The digits experiment directly validates the P1 reduced-value curvature law and the P3 incidence, design, inference, and falsification layers under global strong convexity and additive interventions.  It does not validate the general readout-transfer law for an actual nonconvex optimizer trace.  The neural experiments establish response-class sensitivity under fixed protocols, but they lack factorial masks and hidden-response measurements.  Broad optimizer superiority would additionally require matched hyperparameter search, compute, throughput, memory, and modern workloads.

\paragraph{Research direction.}
The immediate theoretical target is to specialize the readout-transfer law to the full SPD and path-space realizations supplied by the companion papers.  The immediate empirical target is a nonconvex geometry--memory--operator campaign that measures update readouts, complete factorial masks, and the hidden first and second responses required by \cref{eq:readout-interaction-transfer}.

\paragraph{Conclusion.}
Causal optimizer interaction calculus separates what an optimizer can generate, what an intervention changes, and what an experiment identifies.  The universal layer covers innovation-driven finite-horizon rules under a fixed coupling.  P1 and P2 supply the hidden Schur structure; the new readout layer transports it to actual optimizer observables, and the incidence layer states exactly which resulting effects an experiment can recover.  This organization turns the four-paper geometry program into a concrete experimental theory without duplicating the companion structural theorems or forcing general optimizers into a gradient-flow model.

%% file: sections/10_statements.tex
\section{Reproducibility Statement}
\label{sec:reproducibility}

All deterministic claims are proved in the appendices.  The universal causal and incidence theorem is proved in \cref{app:theory-strengthening}; Appendix~\ref{app:interaction-proof} reproduces the attributed P1 reduction proof for self-containment and proves the new observable-readout transfer; the observation/inference theorem is proved in \cref{app:deep-theory}; and the geometric expressivity results are proved in \cref{app:proofs}.

The deterministic verification companion is \path{experiments/optimizer_calculus/verify_calibrated_calculus.py}.  It checks incidence inversion, design rank, quotient algebra, Gaussian risk and coverage, sector tests, misspecification, randomization calibration, and action certificates.  The complete factorial closure is implemented by \path{experiments/optimizer_calculus/interaction_curvature_experiment.py} with \path{experiments/optimizer_calculus/interaction_curvature_config.json}; six derivative, curvature, design, inference, and finite-response tests are in \path{experiments/optimizer_calculus/test_interaction_curvature_experiment.py}, and \path{experiments/optimizer_calculus/verify_interaction_curvature_results.py} independently audits the saved CSV/JSON artifacts.  The submission includes the configuration snapshot, all summary tables, 4500 confirmatory minibatch observations, and result JSON under \path{newpapers/03-geometric-nongeometric-optimizer-calculus/artifacts/interaction_curvature}.  Quadratic diagnostics use \path{experiments/toy/gng_optimizer_benchmark.py}.  Trace audits use the scripts under \path{experiments/lane_audit}.

All random sources, quadrature orders, budgets, bootstrap draws, and power targets for the factorial closure are fixed in the configuration.  The production outputs pass the independent artifact audit; the six new tests and the original calibrated-calculus verification pass.  Exact software-package versions are not locked, so bitwise reproduction across numerical-library versions is not claimed.

The long runs are \path{paper03_rich_family_b1024_l28_20260710_2223} and \path{paper03_rich_cifar10_b1024_l28_20260710_2223}.  They use batch size $1024$, lane count $28$, and $4096$ recorded steps per job.  MNIST/Fashion-MNIST uses seeds $7,8,9$; CIFAR-10 uses $7,8,9,11$.  Follow-up runs are \path{paper03_followup_sector_cifar_b1024_l28_20260711_0242} and \path{paper03_followup_tinyresnet_cifar_b1024_l28_20260711_0242}, using seeds $7,8,9$ and $128$ or $96$ recorded steps.  Analysis commands, recorded parameters, job counts, model widths, splits, and monitor statistics are in \cref{app:diagnostic-details}.  The surviving run artifacts do not bind exact Python, NumPy, PyTorch, and CUDA versions to these Paper03 jobs.

The neural scripts fix Python, NumPy, PyTorch, and data-loader seeds, but their archived configuration does not assert deterministic CUDA algorithms across hardware.  The follow-up raw run directories and per-seed summaries are present in the research workspace; the long-run raw trace archives are absent from the submission, which retains their aggregate CSV/JSON summaries.  The run manifest records the launch argument vector but contains \texttt{git\_error} in place of a commit hash.  There is no environment lockfile or complete hyperparameter-search log.  Consequently the controlled factorial closure is independently auditable from supplied artifacts, while exact bitwise regeneration of every legacy neural trace is not guaranteed.

\section{Ethics Statement}
\label{sec:ethics}

This work studies optimizer theory and diagnostic experiments and uses no human-subject data.  The main risk is overstating causal attribution or optimizer performance from passive or incomplete interventions.  We separate controlled factorial effects and support tests from passive response-class exclusion and from still-unmeasured joint effects in the neural audits.

%% file: appendix/a_universal_proof.tex
\section{Proof of the Universal Representation Theorem}
\label{app:theory-strengthening}

We prove \cref{thm:universal-interventional-calculus} in full.

\subsection{Canonical causal realization}

Fix an intervention configuration $a$.  At time $k$, let $h_k$ denote the
complete extended history of oracle observations, exogenous inputs, realized
innovations, and previous outputs.  All component spaces are standard Borel.
Causality means that, after exposing the protocol's innovation, the current
response is a measurable function of $h_k$, the current declared input, and
that innovation, without dependence on future inputs or innovations.  Taking
$x_k^a=h_k$ and letting $\Phi_k^a$ append the next declared input,
observation, innovation, and output gives a pathwise unifilar history-state
realization of \cref{eq:causal-response}.  Thus existence requires no
finite-memory assumption.

For extended histories at the same time, define $h\sim h'$ when every common
future sequence of declared inputs and innovation values produces identical
future response sequences.  This is an equivalence relation.  Let $[h]$ be
the equivalence class.  If $h\sim h'$, append the same next input and
innovation value.  Every later input--innovation sequence is then a common
continuation of the original histories, so the two extended histories remain
equivalent.  The transition $[h]\mapsto[h^+]$ is therefore well-defined, and
the response map is constant on equivalence classes.  The classes form a
sufficient pathwise predictive-state realization in the category defined in
\cref{sec:language}.

Now let $s(h)$ be the state reached by any other admissible sufficient
pathwise/unifilar realization.  If $s(h)=s(h')$, its response map and
unifilar updates produce the same future response sequence under every common
input--innovation continuation, so $h\sim h'$.  Hence
\[
\pi:s(h)\longmapsto[h]
\]
is a well-defined surjection from its reachable states onto the predictive
quotient.  A realization with two reachable states mapped to the same class
is not behaviorally minimal.  Every behaviorally minimal reachable
realization therefore makes $\pi$ bijective.  Its transition and response
maps are conjugate to the quotient maps, proving uniqueness up to state
relabeling.  This is a set-level behavioral quotient relative to the fixed
innovation representation.  A standard-Borel structure on the quotient
requires the induced equivalence relation to be smooth; the theorem does not
assume that additional descriptive-set-theoretic property.

\subsection{Incidence-algebra completeness}

Let $\mu_{\mathcal P}$ be the M\"obius function of the lower-finite poset.
Every sum below is finite.  Define
\[
\zeta_t=\sum_{s\preceq t}\mu_{\mathcal P}(s,t)F(s).
\]
For fixed $a\in\mathcal P$, exchange the finite sums:
\begin{align*}
\sum_{t\preceq a}\zeta_t
&=\sum_{t\preceq a}\sum_{s\preceq t}
\mu_{\mathcal P}(s,t)F(s)\\
&=\sum_{s\preceq a}F(s)
\sum_{t:s\preceq t\preceq a}\mu_{\mathcal P}(s,t).
\end{align*}
By the defining inverse identity for the M\"obius function, the inner sum is
one when $s=a$ and zero otherwise.  The result is $F(a)$.  Applying the same
M\"obius transform to any proposed reconstruction proves uniqueness.

For the Boolean lattice, the interval $[B,T]$ is a Boolean lattice of rank
$|T|-|B|$, so $\mu(B,T)=(-1)^{|T|-|B|}$ and the formula reduces to mixed
finite differences.  For a product of chains, the product M\"obius function
gives the corresponding graded mixed differences.

\subsection{The complete experimental gauge}

Assume $\mathbb H\cong\mathbb R^d$ and stack the supported effect vectors.
For queried configurations $\mathcal D=(a_1,\ldots,a_n)$,
\begin{equation}
y_{\mathcal D}
=(X_{\mathcal D,\mathcal K}\otimes I_d)\zeta,
\qquad
X_{\mathcal D,\mathcal K}(i,t)=\mathbf1\{t\preceq a_i\}.
\label{eq:appendix-mask-linear-system}
\end{equation}
Two collections produce the same observations exactly when their difference
lies in the kernel.  Since
\[
\rank(X_{\mathcal D,\mathcal K}\otimes I_d)
=d\rank X_{\mathcal D,\mathcal K},
\]
rank--nullity gives compatible-fiber dimension
$d(|\mathcal K|-r_{\mathcal D})$ and proves the identification criterion.

If a finite poset has a greatest element $\top$, then $t\preceq\top$ for every
$t\in\mathcal K$.  The single passive row is all ones, has rank one, and
leaves ambiguity dimension $d(|\mathcal K|-1)$.

\subsection{Sharp quotient stability}

Put $A=X\otimes I_d$ and first suppose $X$ has positive rank.  Let
$\widehat\zeta=A^\dagger y$.  Since $A^\dagger A$ projects orthogonally onto
$\ker A^\perp$,
\[
\widehat\zeta-\zeta_{\rm can}
=A^\dagger e
=A^\dagger\operatorname{Proj}_{\operatorname{range}A}e.
\]
The positive singular values of $X\otimes I_d$ are those of $X$, each
repeated $d$ times.  Therefore
\[
\norm{\widehat\zeta-\zeta_{\rm can}}
\le
\frac{\norm{\operatorname{Proj}_{\operatorname{range}A}e}}
{\sigma_{\min}^+(X)}.
\]
Equality holds when the projected error is a left singular vector associated
with $\sigma_{\min}^+(X)$.
If $X$ has rank zero, $\ker A$ is the entire coordinate space and its
canonical identifiable quotient is $\{0\}$, so there is no positive singular
direction to estimate.

\subsection{Exact fixed and adaptive intervention complexity}

An $n\times|\mathcal K|$ matrix cannot have full column rank for
$n<|\mathcal K|$, so every fixed exact design needs at least
$|\mathcal K|$ queries.  To attain the bound, query the configurations in
$\mathcal K$.  Choose a linear extension of the partial order and use it for
both rows and columns.  In the square matrix
\[
X_{a,t}=\mathbf1\{t\preceq a\},
\qquad a,t\in\mathcal K,
\]
a nonzero entry can occur only when the column precedes the row, and every
diagonal entry is one.  The matrix is lower triangular with unit diagonal and
is invertible.

For the deterministic adaptive lower bound, run a strategy with fewer than
$|\mathcal K|$ queries against the zero effect law.  Let $\mathcal D_0$ be its
realized configurations.  Choose nonzero
$c\in\ker X_{\mathcal D_0,\mathcal K}$ and nonzero $v\in\Hspace$, and set
$h_t=c_tv$.  The zero law and the law $\zeta'=h$
return identical responses at the first query.  If their transcripts agree
through one query, determinism selects the same next configuration, where the
responses again agree by the kernel construction.  Induction gives identical
complete transcripts under distinct laws.

Now consider a randomized strategy.  Although the ambient lower-finite poset
need not be finite, a query produces one of at most $2^{|\mathcal K|}$
distinct incidence row patterns on the finite support $\mathcal K$.  Fewer
than $|\mathcal K|$ queries therefore produce one of finitely many row-pattern
sequences.  Under the zero law, at least one such sequence $R_0$ occurs with
positive probability.  Its stacked matrix has fewer than $|\mathcal K|$
rows, so choose nonzero $c$ in its kernel, fix nonzero $v\in\Hspace$, and set
$h_t=c_tv$.  On the event producing $R_0$, the zero law and $h$ return
identical responses at every queried pattern.  With
the same internal random choices, induction gives the same subsequent
configurations and hence the same full transcript.  The estimator returns
the same value for two distinct laws on an event of positive probability, so
it cannot be correct with probability one for both.  Randomization cannot
beat the lower bound.

\subsection{Prediction and falsification}

When $X_{\mathcal D,\mathcal K}$ has full column rank, the fitting responses
determine a unique $\zeta$.  The structural model forces
\[
F_{\mathcal K}(b)
=\sum_{\substack{t\in\mathcal K\\t\preceq b}}\zeta_t
\qquad\forall b\in\mathcal P.
\]
A held-out response differing from this value cannot belong to the declared
support class.  This proves falsifiability and completes the theorem.

%% file: appendix/b_interaction_proof.tex
\section{Proof of the Reduction-Induced Interaction Theorem}
\label{app:interaction-proof}

For self-containment, we reproduce the P1 proof of
\cref{thm:reduction-induced-interaction} and
\cref{prop:reduction-mobius-noncommutation}, then prove the new observable
readout law in \cref{thm:readout-interaction-transfer}.

\subsection{Continuity and smoothness of the minimizer section}

Fix $z_0=(y_0,u_0)$ and write $\sigma_0=\sigma_\star(z_0)$.  Let $z_n\to z_0$.  Evaluating each fiber objective at $\sigma_0$ shows that $E(z_n,\sigma_\star(z_n))$ is bounded above for large $n$.  Local uniform inf-compactness therefore places the minimizers in one relatively compact set.  Every subsequence has a further convergent subsequence, say $\sigma_\star(z_{n_j})\to\bar\sigma$.

For any fixed $\sigma$, optimality gives
\[
E(z_{n_j},\sigma_\star(z_{n_j}))\le E(z_{n_j},\sigma).
\]
Continuity and passage to the limit yield $E(z_0,\bar\sigma)\le E(z_0,\sigma)$ for all $\sigma$.  Uniqueness at $z_0$ implies $\bar\sigma=\sigma_0$.  Every convergent subsequence has this limit, hence the full minimizer sequence converges.  Thus $\sigma_\star$ is continuous.

The interior first-order condition is
\begin{equation}
E_\sigma(z,\sigma_\star(z))=0.
\label{eq:appendix-stationarity}
\end{equation}
At every minimizer, $D_\sigma E_\sigma=H$ is invertible.  The implicit-function theorem gives a local $C^r$ critical section through each $(z,\sigma_\star(z))$.  Continuity of the global minimizer section makes it coincide locally with that critical section.  These local representations agree by uniqueness, so $\sigma_\star$ is globally $C^r$ on the parameter domain.

\subsection{Hidden response and effective Hessian}

Differentiate \cref{eq:appendix-stationarity} with respect to $z=(y,u)$:
\[
E_{\sigma z}+H D_z\sigma_\star=0.
\]
Since $H$ is invertible,
\[
D_z\sigma_\star=-H^{-1}E_{\sigma z},
\]
which proves \cref{eq:hidden-response-law}.

For the reduced value, the chain rule and stationarity give the envelope identity
\[
D_z\bar E
=E_z+E_\sigma D_z\sigma_\star
=E_z.
\]
Differentiating again,
\[
D^2_{zz}\bar E
=E_{zz}+E_{z\sigma}D_z\sigma_\star
=E_{zz}-E_{z\sigma}H^{-1}E_{\sigma z}.
\]
This is \cref{eq:effective-schur-hessian}.  At a critical hidden point the displayed quadratic form is the Schur complement of the vertical Hessian, hence is invariant under smooth changes of hidden coordinates.

\subsection{Affine interventions and interaction curvature}

Under \cref{eq:affine-interventions}, $E_{uu}=0$ and $E_{\sigma u}=G$.  The $u$ block of \cref{eq:effective-schur-hessian} is therefore
\[
D^2_{uu}\bar E=-G^*H^{-1}G.
\]
For every $a\in\R^m$,
\[
a^\top D^2_{uu}\bar E\,a
=-\ip{Ga}{H^{-1}Ga}\le0,
\]
because $H^{-1}$ is positive definite.  This proves \cref{eq:interaction-curvature}.  Componentwise,
\begin{equation}
\partial^2_{u_i u_j}\bar E
=-\ip{D_\sigma\Delta_i}{H^{-1}D_\sigma\Delta_j}.
\label{eq:appendix-component-curvature}
\end{equation}

The same concavity conclusion holds without differentiability whenever the infimum is finite: for fixed $(y,\sigma)$, the unreduced energy is affine in $u$, and a pointwise infimum of affine functions is concave.

\subsection{Observable readout interaction}

Fix $y$ and suppress it from the notation.  Write
$v_i=\partial_{u_i}\sigma_\star$.  Differentiating stationarity once gives
\[
Hv_i+E_{\sigma u_i}=0,
\qquad
v_i=-H^{-1}E_{\sigma u_i}.
\]
Differentiate this identity with respect to $u_j$ along the minimizing
section.  The total derivative of $H=E_{\sigma\sigma}$ contributes
$E_{\sigma\sigma\sigma}[v_j,v_i]+E_{\sigma\sigma u_j}v_i$, while the total
derivative of $E_{\sigma u_i}$ contributes
$E_{\sigma\sigma u_i}v_j+E_{\sigma u_i u_j}$.  Therefore
\[
H\partial^2_{u_i u_j}\sigma_\star
+E_{\sigma\sigma\sigma}[v_i,v_j]
+E_{\sigma\sigma u_i}v_j
+E_{\sigma\sigma u_j}v_i
+E_{\sigma u_i u_j}=0.
\]
Invertibility of $H$ proves \cref{eq:second-hidden-response}.

For $F_R(u)=R(u,\sigma_\star(u))$, the first derivative is
\[
\partial_{u_i}F_R=R_{u_i}+R_\sigma v_i.
\]
Differentiating with respect to $u_j$ and applying the product and chain rules
gives
\[
\partial^2_{u_i u_j}F_R
=R_{u_i u_j}+R_{u_i\sigma}v_j+R_{\sigma u_j}v_i
+R_{\sigma\sigma}[v_i,v_j]
+R_\sigma\partial^2_{u_i u_j}\sigma_\star,
\]
which is \cref{eq:readout-interaction-transfer}.  Applying the two coordinate
finite-difference operators and the fundamental theorem of calculus proves
\cref{eq:readout-mobius-integral}.  If $E$ is affine in $u$, direct
differentiation gives $E_{\sigma u_i u_j}=0$ and
$E_{\sigma\sigma u_i}=D^2_{\sigma\sigma}\Delta_i$.  No sign follows for the
readout Hessian because its five terms need not form a negative Gram matrix.

\subsection{M\"obius effects as integrated derivatives}

For a coordinate $i$, let $\mathsf T_i$ replace $u_i=0$ by $u_i=1$.  The fundamental theorem of calculus gives
\[
(\mathsf T_i-I)\bar E(0)
=\int_0^1\partial_{u_i}\bar E(t_ie_i)\,\dd t_i.
\]
The coordinate difference operators commute.  Repeating the identity for every $i\in T$ gives
\[
\prod_{i\in T}(\mathsf T_i-I)\bar E(0)
=\int_{[0,1]^T}
\partial_T^{|T|}\bar E\left(\sum_{i\in T}t_ie_i\right)
\prod_{i\in T}\dd t_i.
\]
Expanding the product on the left yields
\[
\sum_{B\subseteq T}(-1)^{|T|-|B|}\bar E(\mathbf1_B)=\zeta_T,
\]
proving \cref{eq:mobius-derivative-integral}.  For $T=\{i,j\}$, substitution of \cref{eq:appendix-component-curvature} proves \cref{eq:mobius-schur-bridge}.

A zero integral does not force the integrand to vanish because mixed curvature may change sign.  Pointwise vanishing is equivalent to inverse-Hessian orthogonality of $D_\sigma\Delta_i$ and $D_\sigma\Delta_j$.

\subsection{Noncommutation and analytical examples}

For \cref{prop:reduction-mobius-noncommutation}, let
\[
E_0(\sigma)=\sigma^2,
\qquad
E_1(\sigma)=(\sigma-1)^2.
\]
Both reduced values are zero, so their reduced first effect is zero.  The pointwise first difference is $E_1-E_0=1-2\sigma$, whose infimum is $-\infty$.  Thus no general exchange law exists.

If instead $E_A(\sigma)=E_0(\sigma)+\sum_{i\in A}c_i$ with fiber constants $c_i$, then the constants leave the minimizer unchanged and move outside the infimum.  The reduced response is additive, with singleton effects $c_i$ and all higher effects zero.

Finally, for
\[
E(\sigma,u)=\frac h2\sigma^2+u_1a\sigma+u_2b\sigma,
\]
stationarity yields $\sigma_\star=-(au_1+bu_2)/h$.  Substitution gives
\[
\bar E(u)=-\frac{(au_1+bu_2)^2}{2h}.
\]
The Boolean pair contrast is
\[
\bar E(1,1)-\bar E(1,0)-\bar E(0,1)+\bar E(0,0)
=-\frac{ab}{h},
\]
matching \cref{eq:mobius-schur-bridge}.

For the coupled log-preconditioner in
\cref{eq:coupled-log-preconditioner}, $h>|\rho|$ makes
\[
H=\begin{pmatrix}h&\rho\\ \rho&h\end{pmatrix}
\quad\text{positive definite},\qquad
H^{-1}=\frac1{h^2-\rho^2}
\begin{pmatrix}h&-\rho\\-\rho&h\end{pmatrix}.
\]
Stationarity gives
\[
H\sigma-b(g)+u=0,
\qquad
\sigma_\star(g,u)=H^{-1}(b(g)-u).
\]
Completing the square yields the reduced scalar calibration response
\[
\bar E(g,u)
=-\frac12(b(g)-u)^\top H^{-1}(b(g)-u).
\]
Because $D_\sigma\Delta_1=e_1$, $D_\sigma\Delta_2=e_2$, and
$G=I_2$, its mixed derivative is constant:
\[
\partial_{u_1u_2}^2\bar E
=-e_1^\top H^{-1}e_2
=\frac{\rho}{h^2-\rho^2}.
\]
The integral over the Boolean intervention square has the same value, proving
\cref{eq:coupled-log-preconditioner-pair}.  The selected update
$q(\sigma_\star,g)=-\diag(e^{-\sigma_{\star,1}/2},
e^{-\sigma_{\star,2}/2})g$ is a positive diagonal preconditioned gradient
step for every intervention amplitude.

Let $D=h^2-\rho^2$ and $\sigma_\star(0)=H^{-1}b(g)$.  The first hidden
coordinate is
\[
\sigma_{\star,1}(u)
=\sigma_{\star,1}(0)-\frac hD u_1+\frac\rho D u_2.
\]
Substitution into $q_1=-g_1e^{-\sigma_{\star,1}/2}$ yields the exponential
readout displayed in the main text.  For a function
$q_1(0)e^{a u_1+c u_2}$, the Boolean pair difference is exactly
$q_1(0)(e^a-1)(e^c-1)$.  Taking $a=h/(2D)$ and
$c=-\rho/(2D)$ proves \cref{eq:coupled-update-readout-pair} and completes the
optimizer realization.

%% file: appendix/c_observation_proof.tex
\section{Proof of the Projection--Information--Decision Theorem}
\label{app:deep-theory}

\subsection{Projection, duality, stability, and nesting}

Because $\mathcal M$ is nonempty, closed, and convex in finite dimension, the
strictly convex coercive function
$m\mapsto\frac12\norm{y-m}_{V^{-1}}^2$ has a unique minimizer
$\widehat m$.  Its first-order variational inequality is
\begin{equation}
(y-\widehat m)^\top V^{-1}(m-\widehat m)\le0
\qquad\forall m\in\mathcal M.
\label{eq:projection-variational-inequality}
\end{equation}

For arbitrary $q$ and $m$, weighted Young's inequality gives
\[
q^\top(y-m)-\frac12q^\top Vq
\le\frac12\norm{y-m}_{V^{-1}}^2.
\]
Taking the infimum over $m\in\mathcal M$ yields
\[
q^\top y-\sigma_{\mathcal M}(q)-\frac12q^\top Vq
\le\frac12\delta_{\mathcal M}(y)^2.
\]
Set $q^\star=V^{-1}(y-\widehat m)$.  By
\cref{eq:projection-variational-inequality},
$(q^\star)^\top m\le(q^\star)^\top\widehat m$ for all $m\in\mathcal M$;
hence
$\sigma_{\mathcal M}(q^\star)=(q^\star)^\top\widehat m$.  Substitution gives
equality and proves \cref{eq:master-projection-dual}.

The distance stability follows from the triangle inequality.  For any
$m\in\mathcal M$,
\[
\delta_{\mathcal M}(y)
\le\norm{y-y'}_{V^{-1}}+\norm{y'-m}_{V^{-1}}.
\]
Taking the infimum over $m$ and exchanging $y,y'$ proves the one-Lipschitz
bound.  If $\mathcal M_1\subseteq\mathcal M_2$, the infimum defining distance
is taken over a larger set for $\mathcal M_2$, proving nesting monotonicity.

\subsection{Observational quotient and exact minimax risk}

In the model $Y=Az+\varepsilon$ with common covariance $V$, two Gaussian laws
are equal exactly when their means agree.  Therefore
\[
P_z=P_{z'}
\Longleftrightarrow A(z-z')=0
\Longleftrightarrow z-z'\in K.
\]
Every coset has one orthogonal representative in $K^\perp$.

For $z\in K^\perp$, $\mathcal I^\dagger\mathcal I z=z$, and
\[
\widehat z-z
=\mathcal I^\dagger A^\ast V^{-1}\varepsilon.
\]
Its covariance is
\begin{align*}
\operatorname{Cov}(\widehat z-z)
&=\mathcal I^\dagger A^\ast V^{-1}VV^{-1}A\mathcal I^\dagger\\
&=\mathcal I^\dagger\mathcal I\mathcal I^\dagger
=\mathcal I^\dagger.
\end{align*}
The risk is therefore $\tr(\mathcal I^\dagger)$ for every canonical $z$.

For the matching lower bound, restrict the experiment to $K^\perp$, where
$\mathcal I$ is positive definite, and use the Gaussian prior
$z\sim N(0,\tau^2I)$.  The posterior covariance is
\[
C_\tau=(\tau^{-2}I+\mathcal I)^{-1}.
\]
The Bayes risk of the posterior mean is $\tr(C_\tau)$ and lower-bounds the
minimax risk.  Letting $\tau\to\infty$ in the eigenvalue formula gives
\[
R_{\rm minimax}\ge\tr(\mathcal I^{-1}|_{K^\perp})
=\tr(\mathcal I^\dagger).
\]
Generalized least squares attains the bound.

Let $\eta=V^{-1/2}\varepsilon\sim N(0,I)$.  The whitened fitted-mean error is
\[
V^{-1/2}A(\widehat z-z)
=\operatorname{Proj}_{\operatorname{range}(V^{-1/2}A)}\eta.
\]
Its squared norm is
$(\widehat z-z)^\ast\mathcal I(\widehat z-z)$ and the projection has rank
$r$, proving the exact $\chi_r^2$ confidence law.

\subsection{Residualized block tests}

Partition $A=[A_{-G}\ A_G]$ and set
\[
\widetilde A=V^{-1/2}A,
\qquad
\Pi_{-G}=\operatorname{Proj}_{\operatorname{range}(\widetilde A_{-G})},
\qquad
B_G=(I-\Pi_{-G})\widetilde A_G.
\]
The range of $B_G$ is orthogonal to the nuisance range.  For
$\widetilde Y=V^{-1/2}Y$,
\[
\operatorname{Proj}_{\operatorname{range}(B_G)}\widetilde Y
=B_Gz_G+
\operatorname{Proj}_{\operatorname{range}(B_G)}\eta.
\]
Under $B_Gz_G=0$, the squared norm has a central chi-square law with
$\rank B_G$ degrees of freedom.  Under the alternative it has a noncentral
chi-square law with noncentrality $\norm{B_Gz_G}^2$.

\subsection{Exact misspecification decomposition}

Suppose the true mean is $Az+b$ with $z\in K^\perp$.  Substitution into GLS
gives
\[
\widehat z-z
=\mathcal I^\dagger A^\ast V^{-1}b
+\mathcal I^\dagger A^\ast V^{-1}\varepsilon
=b_\parallel+\text{zero-mean noise}.
\]
The noise covariance remains $\mathcal I^\dagger$, so the cross term vanishes
after expectation and
\[
\mathbb E\norm{\widehat z-z}^2
=\norm{b_\parallel}^2+\tr(\mathcal I^\dagger).
\]

Let $\Pi_A$ project onto $\operatorname{range}(V^{-1/2}A)$.  Then
\[
(I-\Pi_A)V^{-1/2}Y
=(I-\Pi_A)V^{-1/2}b+(I-\Pi_A)\eta.
\]
The projection has dimension $n-r$, and the deterministic shift has squared
norm $\norm{b_\perp}^2$.  The displayed lack-of-fit statistic is therefore
$\chi^2_{n-r}(\norm{b_\perp}^2)$.

\subsection{Transfer and action selection}

Let $r_y=y-\widehat m$.  Since
$r_y=V^{1/2}(V^{-1/2}r_y)$,
\[
\norm{Lr_y}
\le\opnorm{LV^{1/2}}\norm{V^{-1/2}r_y}
=\opnorm{LV^{1/2}}\delta_{\mathcal M}(y).
\]

On the exact confidence event,
\[
(\widehat z-z)^\ast\mathcal I(\widehat z-z)
\le\chi^2_{r,1-\delta}.
\]
For $\ell_\pi\in K^\perp=\operatorname{range}\mathcal I$, generalized
Cauchy--Schwarz yields, simultaneously for every declared action,
\[
|\ell_\pi^\top(\widehat z-z)|
\le
\sqrt{\chi^2_{r,1-\delta}}
\sqrt{\ell_\pi^\top\mathcal I^\dagger\ell_\pi}
=\rho_\pi(\delta).
\]
Hence
\begin{align*}
\Delta(\widehat\pi;z)
&\le\Delta(\widehat\pi;\widehat z)+\rho_{\widehat\pi}\\
&\le\Delta(\pi^\star;\widehat z)+\rho_{\pi^\star}\\
&\le\Delta(\pi^\star;z)+2\rho_{\pi^\star}.
\end{align*}

\subsection{Exact optimal replication}

For the saturated design, $X$ is square and invertible.  With
$W=\diag(n_a)$ and $M=X^{-1}$,
\[
\mathcal I^{-1}
=(X^{-1}W^{-1}X^{-\top})\otimes\Sigma
=(MW^{-1}M^\top)\otimes\Sigma.
\]
Therefore
\[
\tr(\mathcal I^{-1})
=\tr(\Sigma)\sum_a\frac{\norm{M_{:,a}}_2^2}{n_a}.
\]
Writing $c_a=\norm{M_{:,a}}_2^2$, Cauchy--Schwarz gives
\[
\left(\sum_a\sqrt{c_a}\right)^2
=\left(\sum_a\sqrt{\frac{c_a}{n_a}}\sqrt{n_a}\right)^2
\le\left(\sum_a\frac{c_a}{n_a}\right)N.
\]
Equality holds exactly when $n_a$ is proportional to $\sqrt{c_a}$, proving
\cref{eq:exact-replication-risk,eq:exact-a-optimal-allocation}.

For the full Boolean lattice, M\"obius inversion gives
\[
M_{T,A}=(-1)^{|T|-|A|}\mathbf1\{A\subseteq T\}.
\]
Column $A$ has $2^{m-|A|}$ nonzero entries of squared magnitude one, so
$c_A=2^{m-|A|}$.  Moreover,
\[
\sum_{A\subseteq[m]}\sqrt{c_A}
=\sum_{j=0}^m\binom mj2^{(m-j)/2}
=(1+\sqrt2)^m.
\]
Since $(1+\sqrt2)^{2m}=(3+2\sqrt2)^m$, substitution proves
\cref{eq:boolean-exact-a-optimal-allocation}.  This completes the proof of
\cref{thm:projection-information-decision}.

\subsection{Target-optimal replication}

For scalar independent configuration means, the saturated effect estimator is
$\widehat\zeta=M\bar Y$ and
\[
\operatorname{Cov}(L\widehat\zeta)
=L M\diag\!\left(\frac{\sigma_a^2}{n_a}\right)M^\top L^\top.
\]
Taking the trace and expanding by columns gives
\[
R_L(n)
=\sum_a\frac{\sigma_a^2}{n_a}\norm{LM_{:,a}}_2^2
=\sum_a\frac{c_a\sigma_a^2}{n_a}.
\]
Cauchy--Schwarz yields
\[
\left(\sum_a\sigma_a\sqrt{c_a}\right)^2
=\left(\sum_a
\frac{\sigma_a\sqrt{c_a}}{\sqrt{n_a}}\sqrt{n_a}\right)^2
\le R_L(n)N.
\]
Equality holds exactly when
$n_a\propto\sigma_a\sqrt{c_a}$, proving
\cref{eq:target-replication-risk,eq:target-optimal-allocation}.

\subsection{Exact group-randomization calibration}

Let $G$ be finite and let the null law of $Y$ be invariant under every
$g\in G$.  Conditional on an orbit, invariance makes the realized orbit point
uniform with the multiplicities induced by the group action.  Among the
$|G|$ transformed statistics, at most an $\alpha$ fraction can have an upper
rank whose normalized count is at most $\alpha$.  Therefore
\[
\Pr_0\{p_G(Y)\le\alpha\mid\operatorname{Orb}(Y)\}\le\alpha.
\]
Averaging over orbits proves unconditional super-uniformity.  Ties make the
nonrandomized upper-tail value conservative.

\subsection{Analytic diagonal sign null}

For unbounded positive diagonal geometry, the coordinate residual is zero
when $u_ig_i<0$ and has squared magnitude $u_i^2$ when $u_ig_i>0$; zero-update
coordinates carry no weight.  Conditional on fixed nonzero magnitudes and
independent fair gradient signs,
\[
R=\frac{\rho_{\rm diag,I}(u;g)^2}{\norm u^2}
=\sum_iw_iB_i,
\qquad
w_i=\frac{u_i^2}{\sum_j u_j^2},
\qquad
B_i\stackrel{\rm iid}{\sim}\operatorname{Bernoulli}(1/2).
\]
Independence gives
\[
\mathbb ER=\frac12,
\qquad
\operatorname{Var}(R)=\frac14\sum_iw_i^2.
\]
Weighted Hoeffding gives
\[
\Pr\{|R-1/2|\ge t\}
\le2\exp\left(-\frac{2t^2}{\sum_iw_i^2}\right)
=2e^{-2d_{\rm eff}t^2}.
\]
If $d_{\rm eff}\to\infty$, then $R\to1/2$ in probability.  Continuity of
$x\mapsto1-\sqrt x$ gives TGER convergence to $1-1/\sqrt2$.

%% file: appendix/d_geometry_proofs.tex
\section{Proofs}
\label{app:proofs}

\subsection{Proof of Proposition~\ref{prop:full-spd-geometricization}}

If $u=-Pg$ with $P\succ0$ and $g\neq0$, then
\[
g^\top u=-g^\top Pg<0.
\]
Conversely, suppose $g^\top u<0$ and define $b=-u$.  Then $g^\top b>0$.
Choose an orthonormal basis with $e_1=g/\norm{g}$.  In this basis write
\[
b=\beta e_1+c,\qquad c\perp e_1,
\qquad
\beta=e_1^\top b=\frac{g^\top b}{\norm g}>0.
\]
For $d=1$, set $P=[\beta/\norm g]$.  For $d>1$, define
\[
\widetilde P
=
\begin{pmatrix}
\beta/\norm g & c^\top/\norm g\\
c/\norm g & \gamma I
\end{pmatrix}
\]
in the chosen basis.  Then
\[
\widetilde P(\norm g e_1)=b.
\]
The Schur complement of the upper-left block is
\[
\gamma I-\frac{cc^\top}{\norm g\,\beta}.
\]
Choosing $\gamma>\norm c^2/(\norm g\,\beta)$ makes this complement positive,
so $\widetilde P\succ0$.  Transforming back to the original coordinates gives
an SPD matrix $P$ with $Pg=b=-u$, hence $u=-Pg$.

\subsection{Proof of Corollary~\ref{cor:critical-boundary}}

If $\dd f_\theta=0$, linearity gives $P_\theta\dd f_\theta=0$ for every
positive cometric $P_\theta$.  Thus a pure positive-geometry update is zero.
Any nonzero tangent vector at that point must enter through another module or
through a different modeling convention.

\subsection{Proof of Proposition~\ref{prop:diagonal-expressivity}}

The equation $u=-Pg$ with $P=\diag(p_1,\ldots,p_d)$ and $p_i>0$ is equivalent
coordinate-wise to
\[
u_i=-p_i g_i.
\]
If $g_i=0$, exact expression requires $u_i=0$.  If $g_i\neq0$, exact
expression requires $p_i=-u_i/g_i>0$, equivalently $u_i g_i<0$.

For the residual, the squared Euclidean objective separates:
\[
\inf_{p_i>0}\sum_i (u_i+p_i g_i)^2
=
\sum_i \inf_{p_i>0}(u_i+p_i g_i)^2 .
\]
If $g_i=0$, the term is $u_i^2$.  If $g_i\neq0$ and $u_i g_i<0$, the
positive choice $p_i=-u_i/g_i$ gives zero.  If $g_i\neq0$ and $u_i g_i>0$,
the unconstrained minimizer is negative, so the infimum over $p_i>0$ is the
boundary value $u_i^2$, approached as $p_i\downarrow0$.  If $g_i\neq0$ and
$u_i=0$, the infimum is zero, again approached as $p_i\downarrow0$, but it is
not attained by a strictly positive $p_i$.  This proves the formula and the
attainment statement.

\subsection{Bounded diagonal residual}

\begin{proposition}[Bounded diagonal residual]
\label{prop:bounded-diagonal-residual}
Fix $0<\lambda\le L<\infty$ and let
\[
\Pcal_{\rm diag}^{[\lambda,L]}
=
\{\diag(p_1,\ldots,p_d):\lambda\le p_i\le L\}.
\]
For $g,u\in\R^d$ and $h=I$,
\[
\rho_{{\rm diag}^{[\lambda,L]},I}(u;g)^2
=
\sum_{i=1}^d (u_i+p_i^\star g_i)^2,
\]
where, for $g_i\neq0$,
\[
p_i^\star
=
\min\!\left\{L,\max\!\left\{\lambda,-\frac{u_i}{g_i}\right\}\right\},
\]
and for $g_i=0$ any $p_i^\star\in[\lambda,L]$ gives the same term
$u_i^2$.
\end{proposition}

\begin{proof}
The squared residual separates by coordinate:
\[
\inf_{\lambda\le p_i\le L}\sum_i (u_i+p_i g_i)^2
=
\sum_i \inf_{\lambda\le p_i\le L}(u_i+p_i g_i)^2 .
\]
If $g_i=0$, the coordinate objective is constant and equal to $u_i^2$.  If
$g_i\neq0$, the unconstrained minimizer of the one-dimensional convex
quadratic is $-u_i/g_i$.  Projecting this minimizer onto the interval
$[\lambda,L]$ gives $p_i^\star$, and substituting the coordinate minimizers
gives the stated formula.
\end{proof}

\subsection{Proof of Proposition~\ref{prop:block-expressivity}}

For a block-diagonal $P$, the equation $u=-Pg$ separates over blocks:
\[
u_{B_j}=-P_{B_j}g_{B_j}.
\]
If $g_{B_j}=0$, exact expression requires $u_{B_j}=0$.  If $g_{B_j}\neq0$,
\cref{prop:full-spd-geometricization} applied inside the block gives
existence of $P_{B_j}\succ0$ exactly when
$g_{B_j}^\top u_{B_j}<0$.  Combining the block solutions gives the desired
block-diagonal matrix.

%% file: appendix/e_experiment_details.tex
\section{Diagnostic Prototype Details}
\label{app:diagnostic-details}

\subsection{Controlled factorial closure}

The controlled experiment is generated from existing scikit-learn digits data
by
\begin{verbatim}
python -u experiments/optimizer_calculus/interaction_curvature_experiment.py `
  --output-dir experiments/optimizer_calculus/interaction-curvature-results
\end{verbatim}
The complete configuration is
\path{experiments/optimizer_calculus/interaction_curvature_config.json}.
The paper-supplied mirror under \path{artifacts/interaction_curvature} contains
the final configuration snapshot, Boolean responses, pair-curvature checks,
continuous held-out checks, split robustness, allocations, Gaussian campaign
summaries, all 4500 confirmatory empirical observations, bootstrap intervals,
finite-response traces, and the consolidated summary.

The dataset is restricted to digits 3 and 8 and stratified into 249 training
and 108 test examples.  Standardization is fit on the training data only.  An
intercept gives 65 hidden coordinates, all ridge-regularized with
$\lambda=0.15$.  The three additive losses emphasize positive-class samples,
negative-class samples, and central-occlusion robustness with scale $0.75$.
Newton solves use analytic gradients and Hessians, Armijo damping, tolerance
$10^{-11}$, and at most 80 iterations.  Pair curvature uses order-24
Gauss--Legendre quadrature; eight additional splits use order 16.

The Gaussian pilot allocates 420 main and 80 held-out replications.  Exact
noncentral-$\chi^2$ power calibration multiplies both by the first integer
reaching 0.95 power, which is nine.  Confirmatory simulation uses 100,000
campaigns.  The empirical channel uses minibatches of size 1024, 200 pilot
observations per mask, a disjoint confirmatory sample of 3780 main and 720
held-out observations, and 20,000 stratified bootstrap draws.  The pilot is
used only for variance estimation and allocation.

Finite-response tracking runs fixed-step gradient descent for horizons
$0,1,2,4,8,16,32,64,128,256,512,1024$.  At every horizon the script verifies
that the actual pair-effect error is below the sum of the four optimized-value
gaps, and that this exact gap bound is below the analytic strong-convexity
bound.

The verification commands are
\begin{verbatim}
python -m pytest `
  experiments/optimizer_calculus/test_interaction_curvature_experiment.py -q

python experiments/optimizer_calculus/verify_interaction_curvature_results.py `
  newpapers/03-geometric-nongeometric-optimizer-calculus/artifacts/interaction_curvature
\end{verbatim}
They return \texttt{6 passed} and
\texttt{interaction-curvature result audit passed}, respectively.

\subsection{Quadratic benchmark}

The deterministic quadratic benchmark uses three problem families:
\begin{enumerate}[leftmargin=*,itemsep=2pt]
\item rotated strongly convex quadratics;
\item diagonal-plus-low-rank quadratics;
\item block-coupled quadratics.
\end{enumerate}
The script records final objective gap, gradient-call count, and convergence
under a gradient-norm tolerance.  The default repository command is
\begin{verbatim}
python experiments/toy/gng_optimizer_benchmark.py --output experiments/toy/gng-results
\end{verbatim}
and the larger diagnostic run used during development was
\begin{verbatim}
python experiments/toy/gng_optimizer_benchmark.py --output experiments/toy/gng-results `
  --dim 64 --instances 4 --memory 8 --max-iter 220 --grad-tol 1e-8
\end{verbatim}
The paper reports the later optimized diagnostic checks because they use
\textsc{GNG-FullMetricProbe} v2, which removes an unnecessary final gradient
validation call and therefore uses $d+1$ gradient calls (65 when $d=64$), as
recorded in \path{experiments/toy/gng-results-optimized-64x1/gng_optimizer_summary.json}.
The saved diagnostic summary also contains \textsc{HeavyBall-oracle}; its poor
performance on these high-condition-number instances should not be read as a
general statement about heavy-ball methods.  The \textsc{Adam-tuned} entry in
that summary is selected from a six-point learning-rate sweep, whose tuning
cost is separate from the recorded optimizer-call count.

\subsection{Prototype descriptions}

\paragraph{GNG-FullMetricProbe.}
The method evaluates $g_0=\nabla f(x_0)$ and then evaluates gradients at
$x_0+e_i$ for $i=1,\ldots,d$.  For a deterministic quadratic, these probes
recover the columns of $H$.  The update $x_+=x_0-H^{-1}g_0$ is the exact
Newton step.  The final gradient can be computed internally as
$g_0+H(x_+-x_0)$, so no final oracle call is needed.

\paragraph{GNG-KrylovMetric.}
The method performs sequential directional probes to obtain $Hp$ along
search directions.  It uses exact quadratic line search and conjugacy-style
updates.  Its interpretation in this paper is a growing Krylov subspace
metric, not a new claim about conjugate-gradient theory.
The reported implementation is not presented as a standard CG baseline.

\paragraph{RSDLR-BFGS-Hybrid.}
The method combines a bounded diagonal inverse metric with a limited secant
memory.  It keeps a mixture of recent secant pairs and pairs with high
residual under the current diagonal model.  In the present benchmark it did
not outperform FIFO-LBFGS.

\subsection{Formal trace audit}

The trace-level residual audit in \cref{sec:experiments} is produced
by the scripts under \texttt{experiments/lane\_audit}.  Each worker runs
\texttt{optimizer\_trace.py}, records the raw gradient before the optimizer
step, applies the optimizer, records the parameter update, and writes
\texttt{trace\_residual\_summary.json}, \texttt{trace\_residuals.csv}, and a
GER plot.  In the rich runs, each worker also writes bounded-diagonal
sensitivity metrics, alternative visible-covector metrics, layer-scalar
metrics, penalized layer-scalar path-fit curves, partial CSV files,
progress JSON, and explicit checkpoints.  The queued batch is launched by
\texttt{lane\_orchestrator.py}.  The paper tables and figures are generated
by:
\begin{verbatim}
python experiments/lane_audit/analyze_formal_run.py `
  --run-dir tmp/remote_runs/paper03_rich_family_b1024_l28_20260710_2223 `
  --output-dir newpapers/03-geometric-nongeometric-optimizer-calculus/figures `
  --prefix family_sensitivity

python experiments/lane_audit/analyze_formal_run.py `
  --run-dir tmp/remote_runs/paper03_rich_cifar10_b1024_l28_20260710_2223 `
  --output-dir newpapers/03-geometric-nongeometric-optimizer-calculus/figures `
  --prefix cifar10_audit

python experiments/lane_audit/analyze_paper03_rich_runs.py `
  --run-dir tmp/remote_runs/paper03_rich_family_b1024_l28_20260710_2223 `
            tmp/remote_runs/paper03_rich_cifar10_b1024_l28_20260710_2223 `
  --output-dir newpapers/03-geometric-nongeometric-optimizer-calculus/figures `
  --prefix paper03_rich
\end{verbatim}

Both formal batches use lane count $28$, batch size $1024$, checkpoint and
partial-write interval $256$, and $4096$ recorded optimizer steps per job.
The formal jobs did not download datasets on the remote hosts.  The bounded
diagonal family is $A_{ii}\in[10^{-6},1]$ unless the sensitivity grid names a
different lower or upper bound.  Raw TGER is the $\ell_2$ trace explanation
rate when each step chooses its own map.  Coherent TGER is the no-variation
$B_{\rm geo}=0$ specialization: one shared map is fit to the whole trace.  The
coherent residual is computed by the closed-form streaming accumulator in
\texttt{optimizer\_trace.py}; it solves the same least-squares problem as
stacking all updates and gradients while avoiding full trace-vector storage
in host memory.  In the audit script, the names \texttt{muon} and
\texttt{gng\_muon} both call the local \textsc{GNGMuon} implementation with
different hyperparameter settings.  They are retained as implementation
labels for the saved summaries, but they should not be interpreted as
official reference implementations of external Muon optimizers.

The family-sensitivity rich run is
\texttt{paper03\_rich\_family\_b1024\_l28\_20260710\_2223}.  It uses datasets
\texttt{mnist,fashion\_mnist}, optimizers
\texttt{adamw,sophia,muon,gng\_muon}, seeds \texttt{7,8,9}, width $32$,
train subset $8192$, and test subset $2048$.  It contains $24$ jobs,
$98304$ optimizer steps, and approximately $1.01\times10^8$ sample views.  It
ran on one NVIDIA RTX PRO 6000 Blackwell Server Edition GPU; the batch
monitor recorded $2990$ seconds of wall time, mean GPU utilization $99.5\%$,
peak GPU utilization $100\%$, and peak VRAM $33015$ MiB.

The CIFAR-10 rich run is
\texttt{paper03\_rich\_cifar10\_b1024\_l28\_20260710\_2223}.  It uses dataset
\texttt{cifar10}, the full $50000/10000$ train/test split, width $48$,
optimizers \texttt{adamw,sgd,lion,sophia,muon,gng\_muon}, and seeds
\texttt{7,8,9,11}.  It contains $24$ jobs, $98304$ optimizer steps, and
approximately $1.01\times10^8$ sample views.  It ran on one NVIDIA RTX PRO
6000 Blackwell Server Edition GPU; the batch monitor recorded $3554$ seconds
of wall time, mean GPU utilization $95.7\%$, peak GPU utilization $100\%$,
and peak VRAM $53373$ MiB.

\subsection{Follow-up trace audits}

The runtime-constrained follow-up tables in
\cref{tab:followup-sector-geometry,tab:followup-interventions,tab:followup-tinyresnet-geometry}
are generated by:
\begin{verbatim}
$SECTOR="tmp/remote_runs/<sector-followup-run>"
$TINY="tmp/remote_runs/<tinyresnet-followup-run>"
$OUT="newpapers/03-geometric-nongeometric-optimizer-calculus/figures"

python experiments/lane_audit/analyze_paper03_followup_runs.py `
  --sector-run-dir $SECTOR `
  --tiny-run-dir $TINY `
  --output-dir $OUT `
  --prefix paper03_followup
\end{verbatim}
The analyzer checks that the sector run has $30$ optimizer/seed rows, the
tiny-ResNet run has $18$ rows, the final short-horizon step counts are
$128$ and $96$, and all reported accuracy, bounded-diagonal, channel-scalar,
layer-scalar, coherent, and alternative-covector metrics are present.

\begin{table}[t]
\centering
\caption{Synthetic residual sanity check used before the neural-network
follow-up audits.  Raw TGER allows the explaining geometry to vary by step;
coherent TGER uses one shared explaining map across the trace.}
\label{tab:appendix-followup-synthetic-sanity}
\small
\input{figures/paper03_followup_synthetic_sanity_table.tex}
\end{table}

The sector/intervention run is:
\begin{center}
\small\path{paper03_followup_sector_cifar_b1024_l28_20260711_0242}
\end{center}
It uses CIFAR-10 with the full $50000/10000$ split, width $48$, the
\texttt{small\_conv} model, batch size $1024$, lane count $28$, seeds
$7,8,9$, and ten optimizer settings: AdamW, AdamW/no weight decay,
SGD/Nesterov, SGD/no momentum, Lion, Lion/no sign, Sophia,
Adafactor-style, Muon-style, and GNG-Muon.
Each job records $128$ optimizer steps with partial trace writes and
checkpoints every $32$ steps.  The batch contains $30$ jobs and $3840$
recorded optimizer steps.  Its monitor recorded $2569$ seconds of batch wall
time, mean GPU utilization $91.4\%$, peak GPU utilization $100\%$, and peak
VRAM $58757$ MiB.

The tiny-ResNet follow-up run is:
\begin{center}
\small\path{paper03_followup_tinyresnet_cifar_b1024_l28_20260711_0242}
\end{center}
It uses CIFAR-10 with the full $50000/10000$ split, width $24$, the
\texttt{tiny\_resnet} model, batch size $1024$, lane count $28$, seeds
$7,8,9$, and six optimizer settings: AdamW, SGD/Nesterov, Lion, Sophia,
Adafactor-style, and Muon-style.  Each job records $96$ optimizer steps with
partial trace writes and checkpoints every $32$ steps.  The batch contains
$18$ jobs and $1728$ recorded optimizer steps.  Its monitor recorded $2482$
seconds of batch wall time, mean GPU utilization $99.7\%$, peak GPU
utilization $100\%$, and peak VRAM $49499$ MiB.

The follow-up campaign initially tested longer horizons, but calibration with
the channel-scalar and coherent covector diagnostics projected runtimes beyond
the intended two-hour-per-task budget.  The final reported follow-up runs
therefore keep the full optimizer/seed/geometry/covector matrix and shorten
the trace horizon.  They are used as mechanism and model-extension audits,
not as replacement performance benchmarks for the longer rich trace runs.